\documentclass{article}

\usepackage[main, final]{neurips_2026}

\usepackage[utf8]{inputenc} 
\usepackage[T1]{fontenc}    
\usepackage{hyperref}       
\usepackage{url}            
\usepackage{booktabs}       
\usepackage{array}
\usepackage{amsfonts,amsmath,amsthm,amssymb}       
\usepackage{nicefrac}       
\usepackage{microtype}      
\usepackage{xcolor}         
\usepackage{algorithm}
\usepackage{algpseudocode}
\usepackage{enumerate}
\usepackage{tikz}

\newcommand{\eps}{\varepsilon}

\newcommand{\I}{\mathcal{I}}
\newcommand{\bx}{\mathbf{x}}
\newcommand{\be}{\mathbf{e}}
\newcommand{\tQ}{\texttt{Query}}
\newcommand{\OPT}{\mathrm{OPT}}
\newcommand{\viol}{\mathrm{viol}}

\newtheorem{theorem}{Theorem}[section]
\newtheorem{proposition}[theorem]{Proposition}
\newtheorem{lemma}[theorem]{Lemma}
\newtheorem{corollary}[theorem]{Corollary}
\theoremstyle{definition}

\newtheorem{example}[theorem]{Example}

\theoremstyle{remark}

\title{Learning Lotteries with Minimal Stakeholder Violations from Membership Queries}
\hypersetup{pdftitle={Learning Lotteries with Minimal Stakeholder Violations from Membership Queries},pdfauthor={Nicholas Teh and Paul W. Goldberg}}

\author{%
  Nicholas Teh \\
  University of Oxford \\
  \texttt{nicholas.teh@cs.ox.ac.uk}
  \And
  Paul W.~Goldberg \\
  University of Oxford \\
    \texttt{paul.goldberg@cs.ox.ac.uk}
}

\begin{document}

\maketitle

\begin{abstract}
An AI developer often has several candidate versions of a system and can deploy a mixture of them, sending each request to a randomly chosen version, as in A/B testing. Each stakeholder (e.g., a safety team, a regulator, or a customer with private test data) accepts a mixture exactly when its expected performance (e.g., its rate of harmful replies) meets the stakeholder's own standard, which the developer can learn only by asking whether a proposed mixture is acceptable. We model this as choosing a lottery over $m$ alternatives for $n$ agents, each of whom accepts exactly the lotteries in an unknown halfspace and answers such binary membership queries. While Choo et al. (ICML 2026) sought a lottery that every agent accepts, none may exist when, say, safety and helpfulness requirements conflict, so we minimize the number of rejecting agents. We show that this problem is APX-hard even when all halfspaces are known and all utilities are $0$ or $1$. We then give exact algorithms that run in polynomial time when $m$ is constant, and in quasi-polynomial time when, for a constant $\eta$, some optimal lottery has no accepting agent within $\eta$ of its threshold. When only some agents can be queried, optimizing over a uniform sample of agents gives additive and relative guarantees for the whole population, with a sample size that does not depend on $n$, and a linear program gives a polynomial-time method for any $m$. On the other hand, for $m\ge 3$, a deterministic algorithm that queries only $K$ distinct agents cannot guarantee fewer than $n-K$ extra rejections, and this is achievable; in other words, it can guarantee nothing for the agents it does not ask. Small synthetic experiments indicate that our worst-case bounds are conservative on random instances.
\end{abstract}

\section{Introduction}

Training a language model to refuse harmful requests also makes it refuse some harmless ones \citep{bai2022training,rottger2024xstest}, and the parties affected by a deployment weigh these two errors differently. Consider a chatbot with a \emph{strict} version, which gives a harmful reply to $0.2\%$ of risky prompts but refuses $12\%$ of harmless requests, and a \emph{lenient} version, with rates of $1\%$ and $2\%$. The developer need not commit to one version: it can route each conversation to a randomly chosen version, which is routine in A/B testing and staged releases \citep{kohavi2020trustworthy}. If two thirds of conversations go to the strict version, linearity of expectation gives rates of about $0.47\%$ and $8.7\%$. This mixture meets both a safety team's limit of $0.5\%$ on harmful replies and a hospital's limit of $9\%$ on refusals, although neither version meets both limits on its own. Suppose that a regulator also caps harmful replies at $0.3\%$. Then at least $87.5\%$ of conversations must go to the strict version, whereas the hospital accepts at most $70\%$; no mixture meets all three requirements, and the developer has to choose whose requirement goes unmet.

We model such decisions with a standard framework from probabilistic social choice and random assignment \citep{brandt2017rolling,hylland1979efficient}. A decision maker chooses a \emph{lottery} over a finite set of alternatives (here, model versions), and each agent (here, a stakeholder) accepts it if and only if the expected score of the lottery is at least the agent's threshold; the lotteries that an agent accepts therefore form a halfspace of the probability simplex. For requirements on average rates---of harmful replies, of refusals, or of errors on a customer's own test set---this linear model is exact, because a mixture's rate equals the sum, over versions, of the routing probability times that version's rate. The same model also describes applications ranging from content moderation \citep{gillespie2018custodians,gorwa2020algorithmic} to school or housing allocation \citep{abdulkadirogluSonmez2003schoolChoice,budish2011combinatorial} to the division of a budget among public projects \citep{suksompong2026voting}.

The decision maker, however, knows neither the scores nor the thresholds. A customer will not reveal its private test data, and a regulator's standard may be available only as legal text \citep{euaiact2024}. What a stakeholder can do is look at a concrete proposal and answer yes or no; following \citet{angluin1988queries}, we call such a question (``is this lottery acceptable to you?'') a \emph{membership query}. Answering it demands far less of a stakeholder than reporting full utilities \citep{boutilierEtAl2006constraint,braziunasBoutilier2008elicitation}, yet a single answer can cost an evaluation run or a human review, so a decision maker facing thousands of stakeholders can consult only some of them.

This query model was introduced by \citet{choo2026learnqueries}, who sought \emph{unanimity}: a lottery that every agent accepts, or a report that there is none. Unanimity fits settings in which every requirement is mandatory, but it offers little guidance when requirements conflict. As the example shows, a single demanding agent can rule out every lottery, and learning that no lottery is unanimously acceptable does not tell the developer what to deploy (keeping the old version is one of the alternatives, not a way out of the choice). This paper addresses the following question: \emph{when the requirements cannot all be met, how can a decision maker who asks few membership queries find a lottery that as few agents as possible reject?} We call the problem of minimizing the number of rejecting agents \textsc{Min-Viol}.

One could instead measure how severely each agent is violated, as \citet{teh2026violationdistances} do in the same query model by minimizing a power mean of the distances from the lottery to the violated agents' acceptance boundaries (Appendix~\ref{app:additional_related_work} compares the two objectives). We count rejecting agents because, for safety, a small aggregate violation can hide many agents who each fall just below their thresholds. This choice has a computational cost. Unanimity is a linear feasibility problem, whereas \textsc{Min-Viol} asks for a point of the simplex satisfying as many linear inequalities as possible, which is hard in general (Section~\ref{sec:related}).

\subsection{Our Results}

Our main results concern what can be promised to the stakeholders who are not asked.

\textbf{Even with full information, the problem is hard.}
\textsc{Min-Viol} is APX-hard even when all utilities are $0$ or $1$ (Theorem~\ref{thm:min-viol-apx}), so it has no polynomial-time approximation scheme unless $\mathrm{P}=\mathrm{NP}$; this strengthens the NP-hardness shown by \citet{choo2026learnqueries}. In other words, even a developer who knows every requirement cannot, in general, get arbitrarily close to the best lottery efficiently.

\textbf{Exact algorithms under additional structure.}
When the number $m$ of alternatives is constant (say, a handful of model versions), \textsc{Min-Viol} can be solved exactly in polynomial time (Theorem~\ref{thm:minviol_XP}), by an algorithm that is XP rather than FPT in $m$. For larger $m$, if some optimal lottery leaves no accepting agent within a margin $\eta$ of its threshold, then some optimal lottery uses only $\mathcal{O}(\log n/\eta^2)$ alternatives (Theorem~\ref{thm:few-alternatives}, via a sampling argument of \citet{Castiglioni2023persuasion_maxFSrelaxation}), and searching over such lotteries finds one in quasi-polynomial time for constant $\eta$ (Corollary~\ref{cor:few-alternatives-search}). Finally, when few violations are allowed, Helly's theorem gives an algorithm that is FPT in the numbers of alternatives and violations together, whereas the number of violations alone gives W[2]-hardness (Section~\ref{sec:parameterized}).

\textbf{Learning from a random sample of agents.}
If $k\ge m$ agents are sampled uniformly at random and their halfspaces are recovered, then, with probability at least $1-\delta$, any lottery that minimizes the number of rejecting sampled agents has at most $\OPT+\mathcal{O}(n\sqrt{(m\log(ek/m)+\log(1/\delta))/k})$ violations, where $\OPT$ is the minimum number of violations (Theorem~\ref{thm:budget-additive}). Put another way, for an error that is a given fraction of $n$, the number of agents to ask does not depend on $n$: a developer with a million customers need not ask more of them than one with a thousand. If $\OPT\ge\rho n$, then $\mathcal{O}((m\log(1/\rho)+\log(1/\delta))/(\zeta^2\rho))$ samples give a $(1+\zeta)$-approximation (Theorem~\ref{thm:budget-relative}).

\textbf{A polynomial-time linear program for any number of alternatives.}
For large $m$, we solve a linear program that penalizes agents who reject the lottery or accept it with margin less than $\eta$. With high probability, for every lottery $\bx$, its output has at most $(1+1/\eta)$ times as many violations as $\bx$, plus the number of agents who accept $\bx$ with margin less than $\eta$, plus $\mathcal{O}\big(\frac{n}{\eta}\sqrt{(\log m+\log(1/\delta))/k}\big)$ (Theorem~\ref{thm:shortfall-main}). The guarantee is strong when few agents are barely satisfied.

\textbf{Without random sampling, unasked agents get no guarantee.}
Can clever, adaptive questioning of a fixed set of agents replace random sampling? For $m\ge 3$, it cannot: every deterministic algorithm that queries only $K$ distinct agents may leave $n-K$ more agents unsatisfied than the optimum, however many queries it spends on them (Theorem~\ref{thm:coverage}), and optimizing over a fixed set of $K$ queried agents matches this bound (Theorem~\ref{thm:exact_deterministic_benchmark}). In other words, a deterministic decision maker can promise nothing to the agents it does not ask. For AI deployment, this means that a developer who always consults the same internal teams or large customers has, in the worst case, no guarantee for anyone else; choosing at random whom to ask is what makes the sampling guarantees possible.

\textbf{Synthetic experiments.}
On random instances, our bounds are conservative: with $500$ agents, the lottery learned from a random sample of $120$ had, on average, fewer than $10$ more rejecting agents than the optimum, and the $n-K$ worst case arose only for our lower-bound construction (Section~\ref{sec:experiments}).

\subsection{Related Work}\label{sec:related}
We discuss the closest work here; Appendix~\ref{app:additional_related_work} covers further areas (see Table~\ref{tab:closest-work}).

\textbf{Learning lotteries from membership queries.}
\citet{choo2026learnqueries} bounded the number of membership queries needed to find a lottery that every agent accepts, or to conclude that there is none, and showed that maximizing the number of accepting agents is NP-hard even with binary utilities \citep[Theorem~B.3 of the extended version]{choo2026learnqueries}. We use their routine for recovering one agent's halfspace as a black box, but our objective is different.

\textbf{Maximum feasible subsystem.}
With full information, \textsc{Min-Viol} asks for a point of the simplex satisfying as many of the agents' linear inequalities as possible. It is therefore a restricted case of the \emph{maximum feasible subsystem} (\textsc{Max-FS}) problem, or of its complement, the \emph{minimum unsatisfied linear relation} (\textsc{Min-ULR}) problem.
For general systems of linear inequalities without mandatory constraints, \textsc{Max-FS} can be approximated within a factor $2$ but has no PTAS unless $\mathrm{P}=\mathrm{NP}$ \citep{amaldi1995maxFS}, while \textsc{Min-ULR} is much harder: it cannot be approximated within $2^{\log^{1-c} n}$ for any constant $c>0$ unless $\mathrm{NP}\subseteq\mathrm{QP}$ \citep{arora1997hardness,amaldi1998maxFS}; see also \citet{amaldi2003maxFS} and, for $0/1$ coefficients, \citet{Elbassioni2009approxmaxFS}.
Hardness for general systems does not by itself give hardness for lottery instances, in which the variables form a probability vector, the simplex equation must hold, and, with binary utilities, each agent requires enough probability on a fixed subset of alternatives. Theorem~\ref{thm:min-viol-apx} gives a reduction that stays within these instances. Conversely, a full-information method that treats the simplex constraints as mandatory can be used once all $n$ agents have been recovered, at a cost of $n$ times the per-agent recovery cost; our limited-query results ask what can be guaranteed with fewer queries.
\citet{Castiglioni2023persuasion_maxFSrelaxation} study a relaxation of \textsc{Max-FS} over the simplex in which each constraint may be violated by a small additive slack. Theorem~\ref{thm:few-alternatives} uses their sampling argument, and Section~\ref{sec:few-alternatives} compares the two.

\textbf{AI alignment with diverse stakeholders.}
\citet{conitzer2024social} and \citet{sorensen2024roadmap} argued that aligning AI systems with people whose values differ calls for tools from social choice; this line of work mostly concerns training, whereas we study which mixture of trained versions to deploy. In a similar vein, \citet{neoh2026ensemble} selected a small committee of AI systems from binary or pairwise feedback using few queries. Randomized classifiers in fair classification meet known constraints on group-level rates \citep{agarwal2018reductions}; our constraints are hidden and need not all be met.

\textbf{Probabilistic social choice and portioning.}
Randomization is widely used to aggregate preferences, break ties, allocate resources, or satisfy fairness and efficiency criteria \citep{brandt2017rolling,hylland1979efficient,bogomolnaia2001new,abdulkadirogluSonmez2003schoolChoice,budish2011combinatorial}.
A lottery can also be viewed as a division of a divisible resource, such as money or time, among the alternatives; aggregating preferences into such a division is known as \emph{portioning} or budget aggregation \citep{bogomolnaia2005collective,aziz2019fair,freeman2021truthful,airiau2023portioning,freeman2024project,elkind2026settling}, and we refer to the survey by \citet{suksompong2026voting} for an overview of voting in such divisible settings.
This literature takes preferences, proposed divisions, priorities, or constraints as input, whereas our decision maker observes only accept/reject answers about proposed lotteries.

\textbf{Learning-augmented algorithms.}
Learning-augmented algorithms use possibly erroneous predictions to circumvent worst-case limits such as Theorem~\ref{thm:coverage} \citep{mitzenmacher2021algorithms}, for instance in mechanism design \citep{agrawal2024learning}, online fair division \citep{banerjee2023proportionally,neoh2026online}, and voting \citep{berger2024learning} (Appendix~\ref{app:additional_related_work}); \citet{choo2026learnqueries} used such advice to find a unanimously acceptable lottery with fewer queries.

\section{Preliminaries} \label{sec:prelims}

Let $m\ge 2$ be the number of alternatives and $n\ge 1$ the number of agents. For an integer $z\ge 1$, write $[z]:=\{1,\dots,z\}$. Let $N:=[n]$ be the set of agents and $S:=\{s_1,\dots,s_m\}$ the set of alternatives.
A \emph{lottery} is a probability vector $\bx=(x_1,\dots,x_m)\in\Delta(S):=\{\bx\in\mathbb{R}^m_{\ge 0}:\sum_{j=1}^m x_j=1\}$, where $x_j$ is the probability of $s_j$; the \emph{pure} lottery on $s_j$ is the standard basis vector $\be_j$. We write $\langle\cdot,\cdot\rangle$ for the standard inner product on $\mathbb{R}^m$, and $\log$ for the natural logarithm.

\textbf{Agents.}
Agent $i\in N$ accepts exactly the lotteries in its \emph{acceptable set} $A_i=\{\bx\in\Delta(S):\langle u_i,\bx\rangle\ge\tau_i\}$, where $u_i\in[0,1]^m$ gives agent $i$'s utility $u_i(s_j)$ for each alternative and $\tau_i$ is its threshold. We write $U_i(\bx):=\langle u_i,\bx\rangle$ for the expected utility. Thus $A_i$ is the part of the simplex on one side of a hyperplane (Figure~\ref{fig:simplex}a).
We exclude degenerate agents: every agent accepts at least one lottery and rejects at least one, i.e., $\varnothing\neq A_i\subsetneq\Delta(S)$.
Membership queries reveal $A_i$ but not the pair $(u_i,\tau_i)$: since $\langle\mathbf 1,\bx\rangle=1$ on $\Delta(S)$, the pair $(\lambda u_i+\mu\mathbf 1,\lambda\tau_i+\mu)$ gives the same set $A_i$ for every $\lambda>0$ and $\mu\in\mathbb{R}$. Margins $U_i(\bx)-\tau_i$ are therefore meaningful only after fixing a normalization. If $A_i$ has nonempty relative interior in $\Delta(S)$, we use the unique pair with $\min_j u_i(s_j)=0$ and $\max_j u_i(s_j)=1$. Otherwise, $A_i$ consists of the lotteries that put all probability on some nonempty proper subset of alternatives, and we let $u_i$ be the indicator vector of that subset and $\tau_i=1$ (Lemma~\ref{lem:normalized-representative}). From now on, $(u_i,\tau_i)$ denotes this \emph{normalized pair}; thus $u_i\in[0,1]^m$ and $\tau_i\in(0,1]$.

\textbf{Membership queries.}
A query $(i,\bx)\in N\times\Delta(S)$ returns $\tQ(i,\bx)=1$ if agent $i$ accepts $\bx$ and $0$ otherwise; no utility or threshold is observed. We write $Q$ for the total number of queries and $K$ for the number of distinct agents queried.

\textbf{Finite-grid convention.}
All geometric and statistical results allow arbitrary real-valued normalized pairs. Only when we state an explicit query budget, or a running time that depends on recovered agents, do we assume that the normalized pair lies on a known grid: $u_i(s_j)\in\{0,\eps,2\eps,\dots,1\}$ and $\tau_i\in\{\eps,2\eps,\dots,1\}$, where $\eps\in(0,1/2]$ and $1/\eps$ is an integer. The lottery is not discretized: the decision maker may propose any point of $\Delta(S)$. Under this convention, the halfspace recovery routine of \citet{choo2026learnqueries} finds an agent's normalized pair with $T_\eps=\mathcal{O}(m\log(1/\eps))$ queries; we call this \emph{recovering} the agent. Appendix~\ref{app:finite-precision} explains why some such assumption is needed for exact recovery.

\textbf{Objective.}
The \emph{violation count} of a lottery $\bx$ is $\viol(\bx):=|\{i\in N:U_i(\bx)<\tau_i\}|$, the number of agents who reject $\bx$. The problem \textsc{Min-Viol} asks for a lottery minimizing $\viol$, and $\OPT:=\min_{\bx\in\Delta(S)}\viol(\bx)$.
The minimum is attained: the set of lotteries with at most a given number of violations is a finite union of intersections of the closed sets $A_i$, hence closed in the compact set $\Delta(S)$. The same holds for any finite set or multiset of agents.
The following example shows that mixed lotteries can be strictly better than pure ones.

\begin{example} \label{ex:one-unavoidable-violation}
Let $m\ge 4$ and $n=4$, and let agents $1$, $2$, $3$, and $4$ accept $\bx$ if and only if $x_1+x_2\ge 3/5$, $x_1+x_3\ge 3/5$, $x_2+x_3\ge 3/5$, and $x_4\ge 3/10$, respectively; all utilities are $0$ or $1$. If agents $1$, $2$, and $3$ accept $\bx$, adding their conditions gives $x_1+x_2+x_3\ge 9/10$, so $x_4\le 1/10$ and agent $4$ rejects $\bx$. Hence no lottery is accepted by every agent. The lottery $\bx^*=(1/3,1/3,1/3,0,\dots,0)$ is rejected only by agent $4$, so $\OPT=1$, whereas every pure lottery is rejected by at least two agents. The three agents who accept $\bx^*$ do so with margin $2/3-3/5=1/15$.
\end{example}

\textbf{Interpreting linear acceptability.}
$U_i(\bx)$ is the expected value, under the lottery, of a quantity attached to each alternative: an expected von Neumann--Morgenstern utility \citep{vonNeumannMorgenstern1944,Fishburn1982EU}, or an average outcome such as access probability, service level, safety, or the negative of a loss. By the law of total expectation, such averages combine linearly when an alternative is drawn at random. This linear threshold model is the standard cardinal baseline in random assignment and probabilistic social choice \citep{hylland1979efficient,brandt2017rolling}, and, as with linear classifiers in statistical learning \citep{Vapnik1998,AnthonyBartlett1999}\newline \citep{ShalevShwartzBenDavid2014}, it gives the hidden acceptable sets a finite-dimensional structure that can be learned. Appendix~\ref{app:linear-acceptability} and Section~\ref{sec:conclusion} discuss its limits.

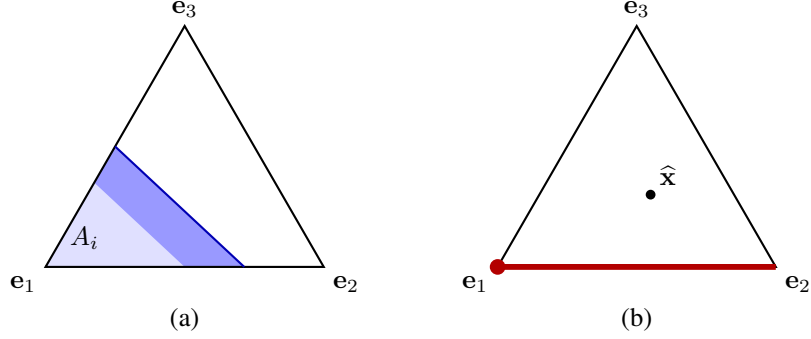
\begin{figure}[t]
\centering
\begin{tikzpicture}[scale=0.92]
\begin{scope}
  \coordinate (A) at (0,0);
  \coordinate (B) at (4,0);
  \coordinate (C) at (2,3.464);
  \coordinate (P) at (barycentric cs:A=0.2857,B=0.7143,C=0);
  \coordinate (R) at (barycentric cs:A=0.5,B=0,C=0.5);
  \coordinate (P2) at (barycentric cs:A=0.5,B=0.5,C=0);
  \coordinate (R2) at (barycentric cs:A=0.65,B=0,C=0.35);
  \fill[blue!12] (A) -- (P) -- (R) -- cycle;
  \fill[blue!40] (P) -- (P2) -- (R2) -- (R) -- cycle;
  \draw[thick] (A) -- (B) -- (C) -- cycle;
  \draw[thick, blue!70!black] (P) -- (R);
  \node[below left] at (A) {$\be_1$};
  \node[below right] at (B) {$\be_2$};
  \node[above] at (C) {$\be_3$};
  \node at (barycentric cs:A=0.8,B=0.08,C=0.12) {$A_i$};
  \node at (2,-0.75) {(a)};
\end{scope}
\begin{scope}[xshift=6.5cm]
  \coordinate (A) at (0,0);
  \coordinate (B) at (4,0);
  \coordinate (C) at (2,3.464);
  \draw[thick] (A) -- (B) -- (C) -- cycle;
  \draw[line width=2.2pt, red!70!black] (A) -- (B);
  \fill[red!70!black] (A) circle (3.2pt);
  \fill (barycentric cs:A=0.3,B=0.4,C=0.3) circle (2pt) node[above right] {$\widehat\bx$};
  \node[below left] at (A) {$\be_1$};
  \node[below right] at (B) {$\be_2$};
  \node[above] at (C) {$\be_3$};
  \node at (2,-0.75) {(b)};
\end{scope}
\end{tikzpicture}
\caption{Lotteries over three alternatives drawn as points of a triangle with corners $\be_1,\be_2,\be_3$. (a) An agent's acceptable set $A_i$ (shaded) is the part of the triangle on one side of a line; the darker strip contains the lotteries that the agent accepts with margin less than $\eta$. (b) The profile in the proof of Theorem~\ref{thm:coverage} when $\widehat\bx\neq\be_1$: every queried agent accepts exactly the lotteries on the bottom edge, and every unqueried agent accepts only $\be_1$. All unqueried agents reject $\widehat\bx$, but every agent accepts $\be_1$.}
\label{fig:simplex}
\end{figure}

\section{Hardness of Approximating \textsc{Min-Viol}} \label{sec:min_viol}

Before asking what can be learned from limited queries, we study the full-information problem, in which all normalized pairs $(u_i,\tau_i)$ are given.

\begin{theorem} \label{thm:min-viol-apx}
    Given all normalized pairs $(u_i,\tau_i)_{i\in N}$, \textsc{Min-Viol} is APX-hard.
    In particular, unless $\mathrm{P} = \mathrm{NP}$, it admits no polynomial-time approximation scheme (PTAS).
    This remains true even when all utilities are binary, i.e., $u_i(s_j) \in\{0,1\}$ for each $i \in N$ and $j \in [m]$.
\end{theorem}

\begin{proof}[Proof idea.]
We give an approximation-preserving reduction from \textsc{Minimum Vertex Cover} on cubic graphs, which is APX-hard \citep{AlimontiKann2000}. For a cubic graph on $q$ vertices, there is one alternative for each vertex and one extra alternative that receives the probability not placed on vertices. The \emph{vertex agent} of vertex $j$ has utility $1$ only on alternative $j$ and threshold $1/q$, so it accepts if and only if $x_j\ge 1/q$. For each edge $\{j,j'\}$, there are $q^2$ identical \emph{edge agents} with utility $0$ on alternatives $j$ and $j'$, utility $1$ elsewhere, and threshold $1-1/q$, so they accept if and only if $x_j+x_{j'}\le 1/q$. A lottery that violates an edge agent violates all $q^2$ copies, more than the $q$ violations of the pure lottery on the extra alternative. From any lottery $\bx$ we obtain a vertex cover of size at most $\viol(\bx)$: all vertices if some edge agent rejects $\bx$, and otherwise the vertices $j$ with $x_j<1/q$, which cover every edge because no edge has both endpoints with probability at least $1/q$ (that would give $x_j+x_{j'}\ge 2/q$). Conversely, a vertex cover $C$ gives a lottery with exactly $|C|$ violations (probability $1/q$ on each vertex outside $C$). Hence an $\alpha$-approximation for \textsc{Min-Viol} gives an $\alpha$-approximation for \textsc{Minimum Vertex Cover} on cubic graphs. The full proof is in Appendix~\ref{app:proof-apx}.
\end{proof}

\textbf{Comparison with hardness proofs for general systems.}
The hardness results of \citet{amaldi1995maxFS} concern linear systems with arbitrary coefficients and unconstrained variables. In our reduction, every graph constraint must instead be an acceptance condition with $0/1$ utilities over lotteries: the edge condition $x_j+x_{j'}\le 1/q$ is written as $1-x_j-x_{j'}\ge 1-1/q$, the probability of all alternatives other than $j$ and $j'$. The extra alternative absorbs the probability not placed on vertices, so that every vertex cover gives a valid lottery and the pure lottery on the extra alternative violates only the $q$ vertex agents; this is what makes violating an edge agent never worthwhile.

\textbf{What the theorem does not show.}
Theorem~\ref{thm:min-viol-apx} rules out a PTAS but not every constant-factor approximation. Much stronger inapproximability is known for \textsc{Min-ULR} on general linear systems \citep{arora1997hardness,amaldi1998maxFS}; we do not know whether it holds for lottery instances. The theorem also concerns multiplicative approximation with full information, so it is not a lower bound matching the additive sampling guarantees of Section~\ref{sec:statistical}.

\section{Exact Algorithms Under Full Information} \label{sec:full_preference_profile}

Theorem~\ref{thm:min-viol-apx} shows that exact optimization needs additional structure. We consider three kinds: a constant number of alternatives (Section~\ref{sec:constant-m}), a margin condition under which few alternatives suffice (Section~\ref{sec:few-alternatives}), and a small number of violations (Section~\ref{sec:parameterized}).

\subsection{A Constant Number of Alternatives}\label{sec:constant-m}
Fix the set of agents that a lottery must satisfy. The lotteries satisfying them form a polytope, and an optimal lottery can be taken to be a vertex of it. Each vertex is the unique solution of the simplex equation together with $m-1$ tight constraints, each of the form $x_j=0$ or $U_i(\bx)=\tau_i$. Enumerating all such linear systems and keeping the best solution gives the following result (proof in Appendix~\ref{app:proof-XP}). It is closely related to exact algorithms for linear programming with a given number of violated constraints in fixed dimension \citep{Matousek1995violations,Chan2005violations}, which can also be applied here; we include it because it is simple and has an explicit running time.

\begin{theorem} \label{thm:minviol_XP}
    Given all normalized pairs $(u_i,\tau_i)_{i\in N}$, there is a deterministic algorithm that returns a lottery minimizing $\viol$ in time $\mathcal{O}(\binom{n+m}{m-1}(m^3+nm))$.
    This is polynomial for every constant number $m$ of alternatives. Since the exponent of $n$ grows with $m$, this is an XP algorithm in the parameter $m$, not an FPT one; in particular, the decision version (is $\OPT$ at most a given number?) is in $\mathrm{XP}$ when parameterized by $m$.
\end{theorem}

\subsection{Few Alternatives Under a Margin Condition}\label{sec:few-alternatives}

A lottery \emph{uses} alternative $s_j$ if $x_j>0$. We now show that a lottery using many alternatives can be replaced by one using few.
For $\eta>0$ and a lottery $\bx$, let $B_\eta(\bx):=\{i\in N:\tau_i\le U_i(\bx)<\tau_i+\eta\}$
be the set of agents who accept $\bx$ with margin less than $\eta$ (in the normalization of Section~\ref{sec:prelims}; see Figure~\ref{fig:simplex}a), and let $s_\eta:=\lceil\log(2n)/(2\eta^2)\rceil$.

\begin{theorem} \label{thm:few-alternatives}
    For every lottery $\bx$ and every $\eta>0$, there is a lottery $\tilde{\bx}$ that uses at most $s_\eta$ alternatives and satisfies $U_i(\tilde{\bx}) \ge U_i(\bx) - \eta$ for all $i \in N$.
    Consequently, $\viol(\tilde{\bx}) \le \viol(\bx) + |B_\eta(\bx)|$. In particular, if some optimal lottery $\bx^*$ has $B_\eta(\bx^*) = \varnothing$, then some optimal lottery uses at most $s_\eta$ alternatives.
\end{theorem}

\begin{proof}[Proof idea and relation to \citet{Castiglioni2023persuasion_maxFSrelaxation}.]
Draw $s_\eta$ alternatives independently from $\bx$ and let $\tilde{\bx}$ be their empirical distribution. For each agent, $U_i(\tilde{\bx})$ is an average of $s_\eta$ independent values in $[0,1]$ with mean $U_i(\bx)$, so Hoeffding's inequality and a union bound over the $n$ agents show that, with probability at least $1/2$, $U_i(\tilde{\bx})\ge U_i(\bx)-\eta$ for every agent. This sampling step is the argument that \citet{Castiglioni2023persuasion_maxFSrelaxation} use for their relaxation of \textsc{Max-FS} over the simplex, and we do not claim it as new. Their guarantee is for the relaxed problem, which counts a constraint as satisfied if it is violated by at most an additive slack. Theorem~\ref{thm:few-alternatives} instead counts violations at the original thresholds: an agent who accepts $\bx$ with margin at least $\eta$ still accepts $\tilde{\bx}$, so the only agents who can become violated are those in $B_\eta(\bx)$. The full proof is in Appendix~\ref{app:proof-few}.
\end{proof}

\begin{proposition}
\label{prop:bounded-search}
    For every integer $s\ge 1$, a lottery minimizing $\viol$ among all lotteries that use at most $s$ alternatives can be computed in time $(m(n+s))^{\mathcal{O}(s)}$.
\end{proposition}

For each set of at most $s$ alternatives, the algorithm applies Theorem~\ref{thm:minviol_XP} to the lotteries that use only those alternatives (proof in Appendix~\ref{app:proof-bounded}). Combining it with Theorem~\ref{thm:few-alternatives} gives the following (proof in Appendix~\ref{app:proof-cor}).

\begin{corollary}\label{cor:few-alternatives-search}
    For every $\eta>0$, running the algorithm of Proposition~\ref{prop:bounded-search} with $s=s_\eta$ takes time $(m(n+s_\eta))^{\mathcal{O}(s_\eta)}$ and returns a lottery $\widehat{\bx}$ with $\viol(\widehat{\bx})\le\viol(\bx)+|B_\eta(\bx)|$ for every lottery $\bx$. In particular, $\viol(\widehat{\bx})\le\OPT+|B_\eta(\bx^*)|$ for every optimal lottery $\bx^*$, and $\widehat{\bx}$ is optimal if $B_\eta(\bx^*)=\varnothing$ for some optimal $\bx^*$.
\end{corollary}

The running time is not polynomial. For constant $\eta$, it is $(mn)^{\mathcal{O}(\log n)}$, which is quasi-polynomial, and it grows quickly as $\eta$ decreases. The lottery $\tilde\bx$ in the proof of Theorem~\ref{thm:few-alternatives} has all probabilities equal to multiples of $1/s_\eta$. Hence the guarantee of Corollary~\ref{cor:few-alternatives-search} is also obtained by the enumeration of \citet{Castiglioni2023persuasion_maxFSrelaxation}, if each of the $\binom{m+s_\eta-1}{s_\eta}$ such lotteries is scored by its exact violation count, in time $\mathcal{O}(\binom{m+s_\eta-1}{s_\eta}nm)$. Proposition~\ref{prop:bounded-search} returns the best lottery among all that use at most $s_\eta$ alternatives, which is never worse, at a higher running time. \citet{Castiglioni2023persuasion_maxFSrelaxation} also show that, under the Exponential Time Hypothesis, their relaxed problem requires quasi-polynomial time.

\subsection{Few Violations}\label{sec:parameterized}
The running time of Theorem~\ref{thm:minviol_XP} is polynomial for constant $m$, but its exponent grows with $m$. When few violations are allowed, the exponential dependence can be limited to $m$ and the number of violations together (proof in Appendix~\ref{app:full-info-fpt}).

\begin{theorem}\label{thm:fpt}
    Given all normalized pairs $(u_i,\tau_i)_{i\in N}$ and an integer $v\in\{0,\dots,n\}$, there is a deterministic algorithm that decides whether some lottery has at most $v$ violations and, if so, returns one, using at most $m^v\cdot \mathrm{poly}(n,m)$ LP feasibility checks over $\Delta(S)$. Hence, when the utilities and thresholds are given explicitly, this problem is fixed-parameter tractable when parameterized by $(v,m)$.
\end{theorem}

\begin{proof}[Proof idea.]

If the acceptable sets of the remaining agents have a common point, an LP finds it. Otherwise, by Helly's theorem \citep{Helly1923} in the $(m-1)$-dimensional affine hull of the simplex, at most $m$ of these sets already have no common point. Such a subfamily can be found with at most $n$ LP feasibility checks, by removing agents one at a time as long as the remaining sets still have no common point. A lottery with at most $v$ violations must violate one of these at most $m$ agents, so the algorithm branches on which of them to drop; the recursion has depth at most $v$ and at most $m$ branches per step. This is in the spirit of algorithms for optimization with few violated constraints \citep{Matousek1995violations}.

The parameter cannot be reduced to the number of violations alone (proof in Appendix~\ref{app:full-info-w2}).
\end{proof}

\begin{theorem}\label{thm:min-viol-w2-hard-k}
    Given all normalized pairs $(u_i,\tau_i)_{i\in N}$ and an integer $v$, deciding whether some lottery has at most $v$ violations is W[2]-hard when parameterized by $v$ alone, even when all utilities are $0$ or $1$. Hence, unless $\mathrm{FPT}=\mathrm{W[2]}$, there is no algorithm with running time $f(v)\cdot\mathrm{poly}(n,m)$ for any function $f$.
\end{theorem}

\begin{proof}[Proof idea.]
We reduce from \textsc{Hitting Set}: given a family of subsets of $[q]$ and an integer $v$, is there a set of at most $v$ elements that meets every member of the family? As in the proof of Theorem~\ref{thm:min-viol-apx}, there is one alternative for each element and one extra alternative. The agent of element $j$ accepts if and only if $x_j\ge 1/q$. For each set $F$ in the family, $v+1$ identical agents accept if and only if $\sum_{j\in F}x_j\le(|F|-1)/q$, which is an acceptance condition with $0/1$ utilities. A lottery with at most $v$ violations cannot violate these $v+1$ copies, so every set $F$ contains an element $j$ with $x_j<1/q$; these elements form a hitting set whose size is the number of violations. Conversely, a hitting set $C$ gives a lottery with exactly $|C|$ violations (probability $1/q$ on each element outside $C$).

Finally, when $m=2$, the simplex is a segment and each acceptable set is an interval, so \textsc{Min-Viol} reduces to finding a point contained in the largest number of intervals, which takes $\mathcal{O}(n\log n)$ time by sorting the endpoints (Proposition~\ref{prop:viol-m2}).
\end{proof}

\section{Limited Queries: Learning from Sampled Agents} \label{sec:statistical}

The previous sections studied the full-information problem, in which every agent's constraint is known. We now ask what can be guaranteed when far fewer than $n$ agents are queried. In this section, $k$ agents $i_1,\dots,i_k$ are sampled independently and uniformly from $N$ (with replacement), and the acceptable sets of the sampled agents are recovered; Section~\ref{sec:deterministic_worst_case} drops the sampling assumption. The results are statistical statements about uniform sampling. Each sampled agent costs at most $T_\eps$ queries, so a budget of $Q$ queries allows $\lfloor Q/T_\eps\rfloor$ samples; Appendix~\ref{app:budget-implementations} states the resulting guarantees together with running times.

\begin{theorem}
\label{thm:budget-additive}
    Fix $\delta \in (0,1)$ and $k \ge m$. With probability at least $1-\delta$, every lottery $\widehat\bx$ that minimizes the number of sampled agents who reject it (counting repeated samples with multiplicity) satisfies $\viol(\widehat \bx) \le \OPT + 8n\sqrt{(m \log(2ek/m) + \log(8/\delta))/k}$.
\end{theorem}

\begin{proof}[Proof idea.]
Agent $i$ rejects $\bx$ if and only if $\langle u_i-\tau_i\mathbf 1,\bx\rangle<0$, so the set of agents rejecting $\bx$ is cut out by a halfspace through the origin of $\mathbb{R}^m$. The family of these sets therefore has VC dimension at most $m$, and the standard VC inequality gives uniform convergence of the sampled rejection rate to the population rate. The argument is standard; what is specific to lotteries is the bound $m$ on the VC dimension. The proof is in Appendix~\ref{app:proof-additive}.
\end{proof}

The same sampling procedure also gives a relative guarantee when $\OPT$ is not too small. The proof replaces the VC inequality with the relative-error bound of \citet{LiLongSrinivasan2001}, which bounds the difference between the sampled and true rejection rates by a small multiple of their sum plus $\rho$ (Appendix~\ref{app:relative-sampling}).

\begin{theorem}
\label{thm:budget-relative}
    There is an absolute constant $C>0$ such that the following holds. Fix $\delta,\rho\in(0,1/2)$ and $\zeta\in(0,1]$, and let $k\ge C(m\log(1/\rho)+\log(1/\delta))/(\zeta^2\rho)$. With probability at least $1-\delta$, every lottery $\widehat\bx$ as in Theorem~\ref{thm:budget-additive} satisfies $\viol(\widehat\bx)\le(1+\zeta/2)\OPT+(\zeta/2)\rho n$. In particular, if $\OPT\ge\rho n$, then $\viol(\widehat\bx)\le(1+\zeta)\OPT$.
\end{theorem}

Theorems~\ref{thm:budget-additive} and~\ref{thm:budget-relative} concern how well a sample represents the population, assuming exact minimization on the sample, which Theorem~\ref{thm:minviol_XP} provides when $m$ is constant. Section~\ref{subsec:shortfall-lp} handles arbitrary $m$.

\subsection{A Polynomial-Time LP for Arbitrary $m$}
\label{subsec:shortfall-lp}

When $m$ is part of the input, exact minimization on the sample is an instance of \textsc{Min-Viol}, which is APX-hard (Theorem~\ref{thm:min-viol-apx}). We therefore replace each violation by a convex penalty. For a margin parameter $\eta\in(0,1]$, agent $i$ is charged $\max\{0,(\tau_i+\eta-U_i(\bx))/\eta\}$: this is $0$ if $i$ accepts $\bx$ with margin at least $\eta$, $1$ on the acceptance boundary, and more than $1$ if $i$ rejects $\bx$. Minimizing the average penalty over the sample is the linear program
\begin{equation}
\label{eq:shortfall-lp}
\begin{aligned}
\text{minimize}\quad & \frac{1}{k}\sum_{t=1}^k z_t \\
\text{subject to}\quad
& z_t \ge \frac{1}{\eta}\left(\tau_{i_t}+\eta-\langle u_{i_t},\bx\rangle\right)
\ \ \text{and}\ \  z_t \ge 0
\quad \text{for all } t\in[k], \quad \bx\in\Delta(S).
\end{aligned}
\end{equation}
Repeated copies of a sampled agent can be merged into one constraint with the corresponding weight in the objective.

\begin{theorem}\label{thm:shortfall-main}
    There is an absolute constant $C>0$ such that the following holds. Fix an integer $k\ge 1$, $\delta\in(0,1)$, and $\eta\in(0,1]$, and let $\widehat \bx_\eta$ be the $\bx$-part of any optimal solution of the linear program \eqref{eq:shortfall-lp}. With probability at least $1-\delta$, $\viol(\widehat \bx_\eta)\le (1+1/\eta)\viol(\bx)+|B_\eta(\bx)|+(Cn/\eta)\sqrt{(\log m+\log(1/\delta))/k}$ for every lottery $\bx$.
\end{theorem}

\textbf{Proof idea.}
The penalty is at least the violation indicator. For a fixed lottery $\bx$, it is at most $1+1/\eta$ for an agent rejecting $\bx$, at most $1$ for an agent in $B_\eta(\bx)$, and $0$ for every other agent, so the population average penalty of $\bx$ is at most $((1+1/\eta)\viol(\bx)+|B_\eta(\bx)|)/n$. The penalty is a $(1/\eta)$-Lipschitz function of $U_i(\bx)-\tau_i$, which is a convex combination of the $m$ functions $i\mapsto u_i(s_j)-\tau_i$; standard Rademacher complexity bounds \citep{BartlettMendelson2002} then give uniform convergence at rate $(1/\eta)\sqrt{\log m/k}$. The proof is in Appendix~\ref{app:proof-lp}.

The bound is useful when some lottery has few violations and few agents accepted with margin less than $\eta$. If $B_\eta(\bx^*)=\varnothing$ for an optimal lottery $\bx^*$, the bound is $(1+1/\eta)\OPT$ plus the sampling error; if some lottery is accepted by every agent with margin at least $\eta$, only the sampling error remains. For instance, in Example~\ref{ex:one-unavoidable-violation}, $\OPT=1$ and the optimal lottery $\bx^*$ given there has $B_{1/15}(\bx^*)=\varnothing$.

\textbf{Full information.}
If the LP is run on all $n$ agents, each once, instead of on a sample, the sampling term disappears: $\viol(\widehat\bx_\eta)\le(1+1/\eta)\viol(\bx)+|B_\eta(\bx)|$ for every lottery $\bx$. The factor $1+1/\eta$ comes from the penalty of violated agents, not from sampling. Capping the penalty at $1$ would remove this factor, but the capped penalty is not convex, so minimizing it is no longer a linear program.

\textbf{Which method to use.}
Exact minimization on the sample, combined with Theorem~\ref{thm:budget-additive}, avoids the factor $1+1/\eta$ and the term $|B_\eta(\bx)|$ of Theorem~\ref{thm:shortfall-main}, and is preferable when it is affordable: when $m$ is a small constant, when $m=2$ (Proposition~\ref{prop:viol-m2}), or when $m$ and the optimal number of violations on the sample are both small (Theorem~\ref{thm:fpt}). Running Corollary~\ref{cor:few-alternatives-search} on the sample takes quasi-polynomial time and is exact only under its margin condition. The LP runs in polynomial time for every $m$ and every sample, without any margin condition, and its sampling term grows with $\log m$ rather than $m$.

\section{The Deterministic Worst Case} \label{sec:deterministic_worst_case}

We now drop the sampling assumption. A \emph{profile} $\I=((u_i,\tau_i))_{i\in N}$ lists all normalized pairs; we write $\viol_\I(\bx)$ and $\OPT(\I)$ for the violation count and its minimum under $\I$. For a deterministic algorithm $\mathcal{A}$, let $\mathcal{A}(\I)$ be the lottery it returns when its queries are answered according to $\I$; its \emph{additive gap} on $\I$ is $\viol_\I(\mathcal{A}(\I))-\OPT(\I)$.

An algorithm that recovers a fixed set $M$ of $K$ agents and returns a lottery $\bx_M$ minimizing the number of violated agents in $M$ has additive gap at most $n-K$ on every profile. Indeed, if $\bx^*$ is optimal for $\I$, then $\bx_M$ violates at most as many agents of $M$ as $\bx^*$ does, and at most $n-K$ agents outside $M$, so $\viol_\I(\bx_M)\le\OPT(\I)+n-K$. The question is whether adaptive queries, possibly many to the same agents, can do better. For $m\ge 3$ they cannot.

\begin{theorem} \label{thm:coverage}
    Fix $m\ge 3$ and a deterministic algorithm $\mathcal{A}$. Run $\mathcal{A}$ with every query $(i,\bx)$ answered ``accept'' if and only if $x_3=\cdots=x_m=0$, and let $M$ be the set of agents it queries and $\widehat{\bx}$ its output. Then there is a profile $\I$ with binary utilities and all thresholds equal to $1$ such that $\mathcal{A}$, run on $\I$, receives the same answers and therefore returns $\widehat{\bx}$, every agent outside $M$ rejects $\widehat{\bx}$, and all agents accept the same pure lottery $\be_1$ or $\be_2$. Hence $\viol_\I(\widehat{\bx}) \ge n-|M|$ and $\OPT(\I)=0$.
\end{theorem}

\begin{proof}[Proof idea.]
Give every agent in $M$ the utility vector $(1,1,0,\dots,0)$ and threshold $1$: such an agent accepts exactly the lotteries with $x_3=\cdots=x_m=0$, which reproduces every answer. Give every agent outside $M$ the utility vector $\be_1$ if $\widehat\bx\neq\be_1$ and $\be_2$ if $\widehat\bx=\be_1$, with threshold $1$. Such an agent accepts only that pure lottery, so it rejects $\widehat\bx$, while that pure lottery is accepted by every agent (Figure~\ref{fig:simplex}b). The answer rule needs a third alternative so that the queried agents reject some lotteries, as nondegeneracy requires, while accepting both $\be_1$ and $\be_2$. The proof is in Appendix~\ref{app:proof-coverage}.
\end{proof}

\begin{theorem}
\label{thm:exact_deterministic_benchmark}
    Assume that each queried agent can be recovered with finitely many membership queries, as under the finite-grid convention, and let $m\ge 3$. For $K\in\{0,\dots,n\}$, let $\mathcal{C}_K$ be the class of deterministic algorithms that query at most $K$ distinct agents on every profile. Then $\inf_{\mathcal{A} \in\mathcal{C}_K}\sup_\I (\viol_\I(\mathcal{A}(\I))-\OPT(\I)) = n-K$, where the supremum is over all profiles satisfying the model assumptions.
\end{theorem}

The upper bound is the argument above, and the lower bound follows from Theorem~\ref{thm:coverage} (proof in Appendix~\ref{app:proof-exact}); \citet{teh2026violationdistances} prove an analogous lower bound for their distance-based objective. In the sampling model of Section~\ref{sec:statistical}, queried agents carry information about unqueried ones. In the worst case they carry none: for $m\ge 3$, the best guarantee depends only on how many distinct agents are queried, not on how many queries are spent on them or how adaptively they are chosen.

\textbf{Bounds after a partial query history.}
The answers observed so far restrict each agent's normalized pair to the pairs consistent with them. For a lottery $\bx$, counting the agents who reject $\bx$ under every consistent pair, and those who reject it under some consistent pair, gives lower and upper bounds on its violation count that hold for every consistent profile; minimizing both over lotteries gives lower and upper bounds on the unknown optimum. The lottery that minimizes the upper count has additive gap at most the difference between the two bounds on the optimum. This difference can only shrink as more answers arrive, and every unqueried agent adds one to it, which is the effect behind Theorem~\ref{thm:coverage}. Appendix~\ref{app:partial-elicitation} proves these facts and gives the exact worst-case additive gap of any fixed lottery after the observed answers.

\section{Synthetic Experiments}\label{sec:experiments}

We ran small synthetic experiments to see how conservative the worst-case bounds are on random instances; Appendix~\ref{app:experiments} gives the setup in full.
We generated $20$ populations with $n=500$ agents and $m=3$ alternatives in each of two settings. In both, $15\%$ of agents reject a planted lottery and $85\%$ accept it, with margin at least $0.05$ in the \emph{separated} setting and less than $0.05$ in the \emph{near-threshold} setting. Since $m=3$, we computed $\OPT$ exactly with the vertex enumeration of Theorem~\ref{thm:minviol_XP}, and checked it against a mixed-integer program on $280$ instances, which agreed in every case. For each population and sample size, we drew $6$ independent samples. Table~\ref{tab:experiments} reports the results.

\begin{table}[h]
\caption{Mean extra rejection rate $(\viol(\widehat\bx)-\OPT)/n$ in percent, with $95$th percentiles in parentheses.}
\label{tab:experiments}
\centering
\begin{tabular}{llcc}
\toprule
Method & Agents used & Separated & Near-threshold \\
\midrule
Exact minimization on the sample & $k=30$  & $6.81$ $(19.24)$ & $6.25$ $(15.81)$ \\
Exact minimization on the sample & $k=120$ & $1.67$ $(4.22)$ & $1.36$ $(3.41)$ \\
Exact minimization on the sample & $k=480$ & $0.43$ $(1.40)$ & $0.23$ $(0.80)$ \\
LP \eqref{eq:shortfall-lp}, $\eta=0.02$ & $k=120$ & $1.96$ $(5.60)$ & $3.60$ $(6.80)$ \\
LP \eqref{eq:shortfall-lp}, $\eta=0.05$ & $k=120$ & $0.53$ $(1.61)$ & $9.20$ $(17.82)$ \\
LP \eqref{eq:shortfall-lp}, $\eta=0.02$ & all $500$ & $0.51$ $(1.67)$ & $1.86$ $(3.80)$ \\
LP \eqref{eq:shortfall-lp}, $\eta=0.05$ & all $500$ & $0.28$ $(0.80)$ & $5.22$ $(8.22)$ \\
\bottomrule
\end{tabular}
\end{table}

\textbf{Findings.}
(i) Exact minimization on $120$ sampled agents is within $2\%$ of the optimum on average, far below the explicit bound of Theorem~\ref{thm:budget-additive}, which exceeds $n$ at these sample sizes. Each fourfold increase in $k$ reduced the mean extra rejection rate by a factor between about $3.9$ and $5.9$, faster than the factor $2$ suggested by the $1/\sqrt{k}$ rate of Theorem~\ref{thm:budget-additive}.
(ii) As the trade-off in $\eta$ of Theorem~\ref{thm:shortfall-main} suggests, $\eta=0.05$ was better than $\eta=0.02$ in the separated setting, and on $120$ samples even better than exact minimization, but more than doubled the extra rejection rate in the near-threshold setting, where $B_{0.05}(\bx)$ contains every agent who accepts the planted lottery $\bx$.
(iii) Querying $K=100$ agents chosen uniformly at random and optimizing over them gave mean additive gaps of $8.46$ and $7.99$ agents, against the worst-case value $n-K=400$; across the tested values of $K$, the mean gap was between about $0.6\%$ and $19\%$ of $n-K$. The construction of Theorem~\ref{thm:coverage} reached a gap of exactly $n-K$ for every tested $K\in\{10,25,50,100,200,400\}$, so the worst case can occur, but in our experiments only for a specially chosen profile.
(iv) With $n=1000$, $m=500$, and $\eta=0.15$, Theorem~\ref{thm:few-alternatives} gives $s_\eta=169$. When every agent accepted the starting lottery with margin at least $\eta$, sampling $169$ alternatives as in its proof never created a new rejection in $480$ trials, and sampling only $50$ created none in $471$ of them. When $20\%$ of agents accepted it with margin less than $\eta$, the mean new rejection rate with $169$ alternatives was $1.19\%$ (95th percentile $1.70\%$), whereas the theorem allows up to $20\%$.

\section{Conclusion}\label{sec:conclusion}
We studied how to choose a lottery that as few agents as possible reject, in the membership query model of \citet{choo2026learnqueries}. The problem is APX-hard even with full information, but it can be solved exactly under additional structure; under uniform sampling, a small number of queried agents gives guarantees for the whole population, and without sampling, for $m\ge 3$, the number of distinct queried agents alone determines the worst-case guarantee of deterministic algorithms.

\textbf{Limitations and extensions.}
The halfspace model covers thresholds on expected access, safety, service, or loss, but not acceptability that depends on variance, ambiguity, or nonlinear probability weighting (Appendix~\ref{app:linear-acceptability}). Theorem~\ref{thm:budget-additive} needs only a bound on the VC dimension of the rejection sets: if another acceptance model gives VC dimension $d$, the same proof applies with $d$ in place of $m$. Theorems~\ref{thm:minviol_XP}, \ref{thm:few-alternatives}, and~\ref{thm:shortfall-main} use linearity directly, through linear boundaries, expected utility as an average over alternatives, and a piecewise-linear penalty, respectively; they do not extend automatically to general convex acceptable sets. The lower bound of Theorem~\ref{thm:coverage} does extend, since halfspaces are a special case of convex sets.

Our results assume noiseless answers. Suppose instead that each answer is flipped with probability $p<1/2$, independently across queries, including repeated ones. Repeating a query $r$ times and taking the majority answer, breaking ties arbitrarily, is wrong with probability at most $\exp(-r(1-2p)^2/2)$ by Hoeffding's inequality. If the noiseless procedure makes at most $Q\ge 1$ queries, then, for $\delta\in(0,1)$, choosing $r=\lceil 2\log(Q/\delta)/(1-2p)^2\rceil$ and stopping after at most $Q$ simulated queries makes every majority answer correct with probability at least $1-\delta$. All noiseless guarantees therefore continue to hold, with the probability that they hold reduced by at most $\delta$ and with at most $rQ$ queries in total. Persistent or adversarial errors, or preferences that change between queries, require a different model.

\begin{ack}
This work was funded by the Advanced Research + Invention Agency (ARIA) under the ASPAI project. PWG is also supported by EPSRC Grants EP/X040461/1 and EP/X038548/1.
\end{ack}

\bibliographystyle{plainnat}
\bibliography{arxiv}

\clearpage

\appendix

\section{Additional Related Work} \label{app:additional_related_work}

Table~\ref{tab:closest-work} compares the works closest to ours. The general objective of satisfying as many linear inequalities as possible is classical; what the table separates is the set over which one optimizes, what information about the constraints is available, and what kind of guarantee is proved.

\begin{table}[ht]
\caption{Comparison with the closest work. ``Explicit'' means that all constraints are given as input.}
\label{tab:closest-work}
\centering
\footnotesize
\renewcommand{\arraystretch}{1.25}
\newcolumntype{L}[1]{>{\raggedright\arraybackslash}p{#1}}
\begin{tabular}{@{}L{0.2\linewidth}L{0.12\linewidth}L{0.14\linewidth}L{0.2\linewidth}L{0.21\linewidth}@{}}
\toprule
Work & Domain & Information & Objective & Guarantees \\
\midrule
\textsc{Max-FS} and \textsc{Min-ULR} \citep{amaldi1995maxFS,amaldi1998maxFS,arora1997hardness,amaldi2003maxFS} & $\mathbb{R}^d$, general linear systems & Explicit & Maximize satisfied, or minimize unsatisfied, constraints & NP-hardness, approximation algorithms, inapproximability \\
\textsc{Max-FS} with $0/1$ coefficients \citep{Elbassioni2009approxmaxFS} & Real variables & Explicit & Maximize satisfied constraints & Approximation algorithms and hardness \\
LP with violations \citep{Matousek1995violations,Chan2005violations} & $\mathbb{R}^d$, $d$ fixed & Explicit & Best linear objective when a given number of constraints may be violated & Exact algorithms in fixed dimension \\
\textsc{Max-FS} with additive slack \citep{Castiglioni2023persuasion_maxFSrelaxation} & Simplex & Explicit & Maximize constraints satisfied up to an additive slack & Quasi-polynomial algorithm; lower bound under ETH \\
Chance-constrained and scenario optimization \citep{calafiore2006scenario,campi2008exact,luedtke2008sample,pagnoncelli2009sample} & Convex or general sets & Samples of random constraints & Optimize a cost subject to a small violation probability & Sample sizes for feasibility of sampled solutions \\
Unanimous lotteries \citep{choo2026learnqueries} & Simplex & Membership queries & A lottery accepted by all agents, or report that none exists & Query upper and lower bounds; NP-hardness of maximizing accepting agents \\
Violation distances \citep{teh2026violationdistances} & Simplex & Membership queries & Minimize a power mean of the distances to violated agents' acceptance boundaries & When queries determine these distances; exact and sampling-based algorithms; lower bound for deterministic algorithms \\
This paper & Simplex & Membership queries (explicit in Sections~\ref{sec:min_viol}--\ref{sec:full_preference_profile}) & Minimize the number of rejecting agents & APX-hardness; exact and parameterized algorithms; sampling guarantees; LP; exact deterministic value $n-K$ \\
\bottomrule
\end{tabular}
\end{table}

\paragraph{Aggregated violation distances.}
\citet{teh2026violationdistances} study the same query model with a different objective. Rather than counting the agents who reject a lottery, they measure how severely each agent is violated: after normalizing each agent's acceptance boundary so that distances are comparable across agents, they take the Euclidean distance from the lottery to the boundary of each violated agent, measured in the affine hull of the simplex, and minimize a power mean of these distances. They characterize when membership queries determine these distances, show that the optimum can then be found exactly, give sampling-based guarantees under a limited query budget, and prove that a deterministic algorithm that leaves some agents unqueried can incur a loss that grows with the fraction of such agents, in parallel with Theorem~\ref{thm:coverage}. The two objectives are complementary. A distance-based objective distinguishes a slight violation from a severe one, but a small aggregate distance can still leave many agents slightly below their thresholds; the violation count treats every rejection as equally serious, which suits requirements that are either met or not. Computationally, the power mean of distances is convex once the relevant boundaries are known, whereas minimizing the violation count is APX-hard even with full information (Theorem~\ref{thm:min-viol-apx}).

\paragraph{Linear programming with violations.}
In fixed dimension, \citet{Matousek1995violations} and \citet{Chan2005violations} study linear programs in which a given number of constraints may be violated, and give exact algorithms whose running time is polynomial in the number of constraints for fixed dimension. Their constraints are explicit and their objective is a linear function. These algorithms are closest in spirit to our full-information results for small $m$ (Theorem~\ref{thm:minviol_XP} and Appendix~\ref{app:full-info}), which concern the violation count over the simplex; our main questions, however, concern what can be learned from membership queries.

\paragraph{Robust fitting and maximum consensus.}
The violation objective also resembles maximum-consensus and robust-fitting formulations, in which one seeks a model consistent with as many constraints or data points as possible \citep{fischler1981random}.
These problems share the view that incompatible constraints should be discarded rather than averaged.
However, robust fitting usually concerns parametric models fitted to observed data, with constraints or residuals available numerically.
In our setting, the unknown constraints are agents' acceptable halfspaces of the simplex, and the decision maker observes only binary answers to lotteries that it proposes. Our deterministic lower bound is therefore not a standard robust-estimation lower bound.

\paragraph{Preference elicitation, utility elicitation, and active querying.}
Our membership query model is motivated by preference and utility elicitation, where full utilities are costly, sensitive, or cognitively difficult to obtain \citep{boutilierEtAl2006constraint,braziunasBoutilier2008elicitation}.
Much of this work uses value queries, comparison queries, bound queries, or minimax-regret criteria to choose good decisions under utility uncertainty. We differ in both the query and the objective. A query in our model asks whether a concrete lottery is acceptable, not for a utility value, a pairwise comparison, or a bound on a parameter. The goal is not to optimize expected utility or regret for a single partially known utility function, but to minimize the number of agents below their individual thresholds.
Recovering sampled agents is one step of our algorithms, but what they output is a lottery with a guarantee for the whole population, not a utility model for every agent.

\paragraph{Membership query learning, active learning, and preference-based learning.}
Membership queries are classical in computational learning theory \citep{angluin1988queries}, and active learning studies how adaptively chosen queries can reduce label complexity \citep{hanneke2014theory}. Preference-based learning and dueling bandits instead infer rewards or rankings from comparisons \citep{bengs2021preference}. Our problem is neither concept learning in the usual sense nor online preference optimization. Each agent has a hidden halfspace of the simplex, but the output is a single lottery that is evaluated on all agents, not a classifier that predicts future labels. This matters for the deterministic lower bound: repeated queries to the same agents may reveal their halfspaces, but they give no information about unqueried agents.

\paragraph{Statistical learning, convex penalties, and chance constraints.}
The sampling results can be viewed as a generalization problem: each lottery $\bx$ defines a binary loss on agents, equal to $1$ if the agent rejects $\bx$ and $0$ otherwise.
The linear acceptability assumption then plays the role of a linear threshold hypothesis class in supervised learning: it limits the possible rejection sets enough for uniform convergence from finitely many sampled agents, without claiming that every acceptability rule is linear \citep{Vapnik1998,AnthonyBartlett1999,ShalevShwartzBenDavid2014}.
The LP replaces the discontinuous violation indicator with a one-sided penalty that is an upper bound on it. Convex upper bounds on discontinuous binary losses are classical \citep{BartlettJordanMcAuliffe2006}, and our analysis uses the behavior of Rademacher complexity under convex hulls and Lipschitz maps \citep{BartlettMendelson2002}. What is specific to lotteries is that the penalty controls the true violation count through the set $B_\eta(\bx)$ of accepted agents with small margin.
Chance-constrained and scenario optimization seek decisions that violate a random constraint with small probability \citep{calafiore2006scenario,campi2008exact,luedtke2008sample,pagnoncelli2009sample}. Our population is finite and accessed through queries rather than given by a distribution, and our deterministic results show that, without a sampling assumption, worst-case guarantees depend only on the number of distinct agents queried.

\paragraph{Learning-augmented algorithms.}
Theorem~\ref{thm:coverage} shows that, without random sampling, nothing can be guaranteed for the agents who are not asked. Learning-augmented algorithms circumvent worst-case limits of this kind in a different way: they receive predictions that may be erroneous, and aim for guarantees that improve with the accuracy of the predictions while remaining reasonable when the predictions are poor \citep{lykouris2021competitive,purohit2018improving,mitzenmacher2021algorithms}. This approach has been applied to mechanism design \citep{xu2022mechanism,balkanski2023strategyproof,agrawal2024learning}, online fair division \citep{banerjee2022online,banerjee2023proportionally,choo2026approximate,cohen2026onlinefairdivisionbudget,huang2026onlinefairdivisionupgrade,neoh2026online,neoh2026closinggapsonlinefair,neoh2026onlinefairdivisionfrequency,teh2026onlinehouseallocationsubsidy}, and voting \citep{berger2024learning}. In particular, \citet{choo2026learnqueries} used advice, such as a promising lottery or the agents whose requirements are likely to be binding, to find a unanimously acceptable lottery with fewer queries. Our guarantees rely on random sampling rather than on predictions, and combining the two is an interesting direction for future work.

\section{Omitted Material from Section~\ref{sec:prelims}}
\subsection{On the linear acceptability assumption}
\label{app:linear-acceptability}

Formally, the model assumes that each agent's acceptable set is a halfspace of the simplex: $A_i=\{\bx\in \Delta(S): \langle u_i,\bx\rangle \ge \tau_i\}$. This is an assumption about the shape of the hidden acceptable sets, not an assumption that agents report numerical utilities. The decision maker observes neither $u_i$ nor $\tau_i$; it only observes whether proposed lotteries are acceptable. Our information model is therefore weaker than elicitation models that query utility values, comparisons, bounds, or regret-relevant parameters \citep{boutilierEtAl2006constraint,braziunasBoutilier2008elicitation}.

\textbf{Risk attitudes.}
The assumption does not require risk neutrality over final outcomes. Suppose alternative $s_j$ induces, for agent $i$, a distribution $P_{ij}$ over consequences, and the agent evaluates consequences with a von Neumann--Morgenstern utility function $v_i$ \citep{vonNeumannMorgenstern1944,Fishburn1982EU}.
Under the standard expected utility treatment of compound lotteries \citep{AnscombeAumann1963}, if the decision maker first draws $j\sim \bx$ and then deploys $s_j$, the agent's expected utility is $\sum_{j=1}^m x_j\,\mathbb{E}_{c\sim P_{ij}}[v_i(c)]$, so the expected consequence utilities, normalized as in Section~\ref{sec:prelims}, give exactly $U_i(\bx)=\langle u_i,\bx\rangle$.
Risk attitudes over consequences are part of $v_i$; for example, if consequences are monetary, a concave $v_i$ represents risk aversion. The linearity is only in the decision maker's random choice among finitely many alternatives.

\textbf{Scores.}
In many platform, allocation, and public policy settings, the alternatives are deterministic rules, and the lottery describes randomized deployment across cohorts, applicants, time periods, or cases. If agent $i$'s requirement concerns an average quantity, oriented so that larger is better, such as expected access probability, service level, safety, compliance, or the negative of an expected loss, then the value of a mixture is exactly the convex combination of the values of its components. Here $u_i(s_j)$ is not a psychological utility but a score of alternative $s_j$, and $\tau_i$ is the minimum acceptable score. For example, if $u_i(s_j)=1$ for safe alternatives and $0$ for unsafe ones, the requirement that an unsafe alternative is chosen with probability at most $1/10$ is $\langle u_i,\bx\rangle\ge 9/10$.
This expected score model is the standard cardinal baseline for evaluating lotteries in random assignment and probabilistic social choice \citep{hylland1979efficient,brandt2017rolling}, although that literature also studies ordinal preferences and other ways of comparing lotteries.

\textbf{A learning-theoretic view.}
Each lottery $\bx$ defines a rejection region $\{(u,\tau):\langle u,\bx\rangle<\tau\}$ in the space of agent types. These regions form a finite-dimensional family of linear threshold regions, analogous to linear classifiers in statistical learning \citep{Vapnik1998,AnthonyBartlett1999,ShalevShwartzBenDavid2014}, which is what makes uniform convergence from sampled agents possible. Richer acceptability models give different families of regions and require different sample complexity, query complexity, and optimization analyses.

\textbf{Scope.}
The model does not capture agents whose acceptability depends on variance, nonlinear probability weighting, ambiguity about the randomization, regret, or other departures from expected utility \citep{Markowitz1952,KahnemanTversky1979,Quiggin1982,Machina1982,GilboaSchmeidler1989}, nor requirements on the distribution of realized outcomes that cannot be written as a linear threshold. For example, an agent who trades off the mean and variance of realized outcomes generally has a nonlinear acceptable set. Section~\ref{sec:conclusion} explains which of our results extend beyond halfspaces.
The linear model is nevertheless not trivial: by Theorem~\ref{thm:min-viol-apx}, minimizing violations is APX-hard even with full information and binary utilities, so the difficulty does not come from complicated utilities or nonlinear behavior.

\subsection{Normalization and finite precision}
\label{app:finite-precision}

As noted in Section~\ref{sec:prelims}, membership queries identify an acceptable set $A_i$, not a pair $(u_i,\tau_i)$, so margins $U_i(\bx)-\tau_i$ are meaningful only after fixing a normalization. The normalization of Section~\ref{sec:prelims} is justified by the following elementary fact.

\begin{lemma}
\label{lem:normalized-representative}
Let $A\subsetneq\Delta(S)$ be nonempty and suppose $A=\{\bx\in\Delta(S):\langle v,\bx\rangle\ge \theta\}$ for some nonconstant $v\in[0,1]^m$ and some $\theta\in\mathbb{R}$. Then exactly one of the following holds.
\begin{enumerate}[(i)]
    \item $A$ has nonempty relative interior in $\Delta(S)$. In this case
    there is a unique pair $(u,\tau)$ with $A=\{\bx\in\Delta(S):\langle u,\bx\rangle\ge\tau\}$, $\min_{j\in[m]}u(s_j)=0$, and $\max_{j\in[m]}u(s_j)=1$.
    \item $A$ has empty relative interior in $\Delta(S)$. In this case $A=\Delta(T):=\{\bx\in\Delta(S):x_j=0\text{ for all }j\notin T\}$ for a unique nonempty proper subset $T\subsetneq[m]$, and $A$ is also given by $u(s_j)=\mathbf 1\{j\in T\}$ and $\tau = 1$.
\end{enumerate}
In both cases, the normalized pair gives exactly the same accept/reject answers as the original pair.
\end{lemma}

\begin{proof}
Let $v_{\min}:=\min_j v(s_j)$ and $v_{\max}:=\max_j v(s_j)$. Since $A$ is nonempty and not all of $\Delta(S)$, we have $v_{\min}<\theta\le v_{\max}$.

If $\theta=v_{\max}$, then a lottery is accepted if and only if it puts all probability on alternatives attaining the maximum score, so $A=\Delta(T)$ with $T:=\{j\in[m]:v(s_j)=v_{\max}\}$. Since $v$ is nonconstant, $T$ is a nonempty proper subset of $[m]$, so $A$ has empty relative interior. Moreover, $T$ is determined by $A$, as it is the set of $j$ with $\be_j\in A$, and the indicator pair in (ii) gives exactly the set $A$.

Suppose now that $\theta<v_{\max}$. Then $A$ contains all lotteries close enough to a pure lottery $\be_j$ with $v(s_j)=v_{\max}$, so it has nonempty relative interior. Let $(u',\tau')$ be any pair with $A=\{\bx\in\Delta(S):\langle u',\bx\rangle\ge\tau'\}$, and consider the affine functions $f(\bx):=\langle v,\bx\rangle-\theta$ and $f'(\bx):=\langle u',\bx\rangle-\tau'$ on the affine hull of $\Delta(S)$. Since $f$ is positive near a pure lottery $\be_j$ with $v(s_j)=v_{\max}$ and negative near one with $v(s_j)=v_{\min}$, it vanishes at some point $\mathbf{p}$ in the relative interior of $\Delta(S)$. Near $\mathbf{p}$, both $\{f\ge 0\}$ and $\{f'\ge 0\}$ coincide with $A$; hence $f'$ is not constant, the zero sets of $f$ and $f'$ agree near $\mathbf{p}$ and therefore coincide, and both functions are nonnegative on the same side. Thus $f'=\lambda f$ on the affine hull for some $\lambda>0$, and since $\langle\mathbf 1,\bx\rangle=1$ there, we have $u'=\lambda v+\mu\mathbf 1$ and $\tau'=\lambda\theta+\mu$ for some $\mu\in\mathbb{R}$. The conditions $\min_j u'(s_j)=0$ and $\max_j u'(s_j)=1$ hold exactly when $\lambda=1/(v_{\max}-v_{\min})$ and $\mu=-\lambda v_{\min}$; conversely, with these values, the pair $(\lambda v+\mu\mathbf 1,\lambda\theta+\mu)$ gives $A$ and satisfies both conditions. Hence the normalized pair exists and is unique.
\end{proof}

The finite-grid convention of Section~\ref{sec:prelims} is used only for explicit query budgets and running times; the geometric, statistical, and lower-bound results hold for arbitrary real-valued normalized pairs unless they explicitly invoke it. There are two reasons to make it explicit.
First, exact recovery of real-valued coefficients from finitely many binary answers is impossible without a finite description or a tolerance. This already happens when $m=2$: if an agent accepts $\lambda \be_1+(1-\lambda)\be_2$ if and only if $\lambda\ge a$, a query at $\lambda$ returns only whether $\lambda\ge a$. There are only countably many finite binary answer sequences, while $a$ ranges over a continuum.
Second, the grid gives every recovered pair a finite description, so the algorithms that use recovered agents have well-defined running times.
We use the per-agent bound $T_\eps= \mathcal{O}(m\log(1/\eps))$ of \citet{choo2026learnqueries} as a black box. The convention does not discretize the lottery: the output may be any point of $\Delta(S)$.

\section{Proof of Theorem~\ref{thm:min-viol-apx}}\label{app:proof-apx}
    We give a polynomial-time approximation-preserving reduction from \textsc{Minimum Vertex Cover} on cubic graphs, which is APX-hard \citep{AlimontiKann2000}. The idea is that vertex agents ask for probability at least $1/q$ on their vertex, while many copies of each edge agent forbid this much probability on both endpoints of an edge; the vertices that do not receive probability $1/q$ then form a vertex cover.

    Let $G=(V,E)$ be a cubic graph with vertices $V=\{v_1,\dots,v_q\}$. Set $m := q+1$, where $s_j$ corresponds to $v_j$ for $j \in [q]$ and $s_m$ is an extra alternative. Create $q$ vertex agents $a_1,\dots,a_q$ and, for each edge $e \in E$, $q^2$ copies $b_{e,1},\dots,b_{e,q^2}$ of an edge agent. Then $n = q + q^2 |E| = \mathcal{O}(q^3)$, since $|E| = 3q/2$. The utilities and thresholds are as follows.
    \begin{itemize}
        \item For each $j\in[q]$, the \emph{vertex agent} $a_j$ has threshold $1/q$, utility $1$ on $s_j$, and utility $0$ on every other alternative. Thus $U_{a_j}(\bx)=x_j$, and $a_j$ accepts $\bx$ if and only if $x_j\ge 1/q$.
        \item For each edge $e=\{v_j,v_{j'}\}\in E$, every copy $b_{e,r}$ has threshold $1-1/q$, utility $0$ on $s_j$ and $s_{j'}$, and utility $1$ on every other alternative. Thus $U_{b_{e,r}}(\bx)=1-x_j-x_{j'}$, and $b_{e,r}$ accepts $\bx$ if and only if $x_j+x_{j'}\le 1/q$.
    \end{itemize}
    Every utility vector lies in $\{0,1\}^m$ and takes both values $0$ and $1$, and every threshold lies in $(0,1)$ (a cubic graph has $q \ge 4$ vertices), so all pairs are normalized.
    Every agent is nondegenerate: $a_j$ accepts $\be_j$ and rejects $\be_m$, while each edge agent of $e$ accepts $\be_m$ and rejects the pure lotteries on the endpoints of $e$.
    Let $\mathrm{vc}(G)$ be the minimum size of a vertex cover of $G$. We show that $\OPT=\mathrm{vc}(G)$ for the constructed instance and that every lottery can be turned into a vertex cover of size at most its violation count.

    \textbf{$\OPT\le \mathrm{vc}(G)$.}
    Let $C\subseteq V$ be a vertex cover, and define $\bx^C\in\Delta(S)$ by $x^C_j=1/q$ if $v_j\notin C$, $x^C_j=0$ if $v_j\in C$, and $x^C_m:=1-(q-|C|)/q=|C|/q$.
    For every edge, at most one endpoint lies outside $C$, so the probabilities of its endpoints sum to at most $1/q$ and all its edge agents accept $\bx^C$. The vertex agent $a_j$ accepts $\bx^C$ if and only if $v_j\notin C$. Therefore $\viol(\bx^C)=|C|$, and minimizing over $C$ gives $\OPT\le \mathrm{vc}(G)$.

    \textbf{From lotteries to vertex covers.}
    Given a lottery $\bx$, let $C(\bx) := V$ if $x_j + x_{j'} > 1/q$ for some edge $\{v_j,v_{j'}\} \in E$, and $C(\bx) := \{v_j \in V : x_j < 1/q\}$ otherwise; $C(\bx)$ can be computed in polynomial time.
    In the first case, $\bx$ violates all $q^2$ copies of some edge agent, so $|C(\bx)| = q \le q^2 \le \viol(\bx)$, and $C(\bx)=V$ is a vertex cover.
    In the second case, no edge has both endpoints with probability at least $1/q$, since that would give $x_j+x_{j'}\ge 2/q$. Hence $C(\bx)$ is a vertex cover. No edge agent is violated, and the violated vertex agents are exactly those of the vertices in $C(\bx)$, so $|C(\bx)| = \viol(\bx)$.
    In both cases, $C(\bx)$ is a vertex cover with $|C(\bx)|\le\viol(\bx)$. In particular, $\mathrm{vc}(G)\le\OPT$, so $\OPT=\mathrm{vc}(G)$.

    \textbf{Approximation preservation.}
    If $\bx$ is an $\alpha$-approximate solution of the constructed instance, then $|C(\bx)| \le \viol(\bx) \le \alpha \cdot \OPT = \alpha \cdot \mathrm{vc}(G)$, so $C(\bx)$ is an $\alpha$-approximate vertex cover of $G$. Hence a PTAS for \textsc{Min-Viol} would give a PTAS for \textsc{Minimum Vertex Cover} on cubic graphs, which does not exist unless $\mathrm{P}=\mathrm{NP}$. Therefore \textsc{Min-Viol} is APX-hard, even with binary utilities.

\section{Omitted Proofs from Section~\ref{sec:full_preference_profile}}

\subsection{Proof of Theorem~\ref{thm:minviol_XP}}\label{app:proof-XP}
    The idea is that some optimal lottery is a vertex of the polytope of lotteries accepted by the agents that it satisfies, and every such vertex is the unique solution of a system of $m$ linear equations; enumerating these systems therefore finds an optimal lottery.
    For any $N' \subseteq N$, let $P(N') := \{\bx \in \Delta(S) : \langle u_i, \bx\rangle \ge \tau_i \text{ for all } i \in N'\} = \bigcap_{i\in N'} A_i$, with the convention $P(\varnothing) = \Delta(S)$.
    We first prove two lemmas.
    \begin{lemma} \label{lem:viol-vertex-opt}
        Let $\bx^*\in\arg\min_{\bx\in\Delta(S)} \viol(\bx)$ and $N^* := \{ i\in N : \bx^* \in A_i \}$.
        Then there exists an extreme point $\bar \bx$ of $P(N^*)$
        such that $\viol(\bar \bx)=\viol(\bx^*)$.
    \end{lemma}
    \begin{proof}
        Since $\bx^* \in P(N^*)$, the set $P(N^*)$ is nonempty. 
        By the definition of $A_i$, $P(N^*) = \Delta(S)\cap \bigcap_{i\in N^*}\{\bx\in\mathbb R^m:\langle u_i,\bx\rangle\ge \tau_i\}$, so $P(N^*)$ is a compact polytope.
        Thus, $P(N^*)$ has at least one extreme point;
        fix any extreme point $\bar \bx \in P(N^*)$.
        By definition of $P(N^*)$, we have $\langle u_i,\bar \bx \rangle \geq \tau_i$ for every $i\in N^*$,
        so $\bar \bx$ violates at most the agents in $N\setminus N^*$. Therefore, $\viol(\bar \bx) \leq |N\setminus N^*| = \viol(\bx^*)$.
        Since $\bx^*$ is optimal, equality must hold and $\bar \bx$ is also optimal.
    \end{proof}

    \begin{lemma} \label{lem:viol-active-set}
        Let $N'\subseteq N$ and let $\bx$ be an extreme point of $P(N')$.
        Then there exist subsets $J\subseteq [m]$ and $N'_0\subseteq N'$ with $|J|+|N'_0|=m-1$ such that $\bx$
        is the unique solution to the linear system $\sum_{j=1}^m z_j=1$, $z_j = 0$  for all $j \in J$, and $\langle u_i,\mathbf{z}\rangle=\tau_i$ for all $i \in N'_0$.
    \end{lemma}
    
    \begin{proof}
        Let $L:=\{\mathbf{d}\in\mathbb{R}^m : \sum_{j=1}^m d_j = 0\}$ be the translation space of the affine hull of $\Delta(S)$, where $\dim(L)=m-1$.
        
        Let $T_0:=\{j\in [m] : x_j=0\}$ and $T_1:=\{i\in N' : \langle u_i,\bx\rangle=\tau_i\}$ be the sets of tight nonnegativity constraints and tight agent constraints at $\bx$, respectively. Consider the family $\mathcal{F} := \{ \mathbf{d}\mapsto d_j : j\in T_0\} \cup \{\mathbf{d}\mapsto \langle u_i, \mathbf{d}\rangle : i\in T_1\}$ of linear functionals on $L$.
        
        We claim that $\mathcal{F}$ spans $L^*$. 
        Indeed, suppose not. Then there
        exists a nonzero $\mathbf{d}\in L$ such that $d_j=0$ for all $j \in T_0$ and $\langle u_i, \mathbf{d}\rangle=0$ for all $i \in T_1$.
        
        For each $j \notin T_0$, we have $x_j > 0$; and for each $i \in N' \setminus T_1$, we have
        $\langle u_i, \bx\rangle > \tau_i$. Since there are only finitely many such indices, there exists
        $\alpha_0 > 0$ such that, for every $|\alpha| \le \alpha_0$, we have $x_j + \alpha d_j \ge 0$ for all $j \notin T_0$ and $\langle u_i, \bx + \alpha \mathbf{d}\rangle \ge \tau_i$ for all $i \in N' \setminus T_1$.
        For $j \in T_0$, we have $x_j + \alpha d_j = 0$. For $i \in T_1$, we have $\langle u_i, \bx + \alpha \mathbf{d}\rangle = \langle u_i, \bx\rangle + \alpha \langle u_i, \mathbf{d}\rangle = \tau_i$.
        Also, $\sum_{j=1}^m (x_j + \alpha d_j) = 1$ because $\mathbf{d} \in L$.
        Thus, $\bx + \alpha \mathbf{d} \in P(N')$ for every $|\alpha| \le \alpha_0$.
        Since $\mathbf{d} \neq 0$, this gives us a nontrivial line segment in $P(N')$ through $\bx$, contradicting the fact that $\bx$ is an extreme point.
        Therefore $\mathcal{F}$ spans $L^*$.
        Choose from $\mathcal{F}$ a basis $\{\mathbf{d}\mapsto d_j : j\in J\} \cup \{ \mathbf{d} \mapsto \langle u_i, \mathbf{d}\rangle : i\in N'_0\}$ of $L^*$, where $J\subseteq T_0$, $N'_0\subseteq T_1$, and $|J|+|N'_0|=m-1$.

        We now show that $\bx$ is the unique solution to the system $\sum_{j=1}^m z_j = 1$, $z_j=0$ for all $j \in J$, and $\langle u_i, \mathbf{z}\rangle=\tau_i$ for all $i\in N'_0$.
        Clearly $\bx$ satisfies this system.
        Conversely, if $\mathbf{z}$ is any solution, then
        $\mathbf{d} := \mathbf{z} - \bx \in L$, $d_j=0$ for all $j \in J$, and $\langle u_i, \mathbf{d} \rangle=0$ for all $i\in N'_0$.
        Since the chosen functionals form a basis of $L^*$, this forces $\mathbf{d}=\mathbf{0}$, and thus $\mathbf{z}=\bx$. Therefore, $\bx$ is the unique solution.
    \end{proof}
    
    Now, for each pair $(J, N_0)$ with $J \subseteq [m]$, $N_0 \subseteq N$, and
    $|J| + |N_0| = m - 1$, let $\mathcal{L}(J, N_0)$ denote the linear system $\sum_{j=1}^m x_j = 1$, $x_j = 0$ for all $j \in J$, and $\langle u_i, \bx\rangle = \tau_i$ for all $i \in N_0$.
    Lemma~\ref{lem:viol-active-set} shows that every extreme point of every polytope $P(N')$ is the unique
    solution of $\mathcal{L}(J,N_0)$ for some such pair $(J,N_0)$. This motivates the following algorithm (Algorithm~\ref{alg:minviol-full-profile}).
    
    \begin{algorithm}[t]
    \caption{Returns a lottery with the minimum number of violations.}
    \label{alg:minviol-full-profile}
    \textbf{Input:} The normalized pairs $(u_i,\tau_i)_{i\in N}$.\\
    \textbf{Output:} A lottery $\widehat\bx$ minimizing $\viol$.
    \begin{algorithmic}[1]
    \State $\widehat\bx \leftarrow \be_1$
    \For{each pair $(J,N_0)$ with $J \subseteq [m]$, $N_0 \subseteq N$, and $|J|+|N_0|=m-1$}
        \If{$\mathcal{L}(J,N_0)$ has a unique solution $\bx$, $\bx \in \Delta(S)$, and $\viol(\bx)<\viol(\widehat\bx)$}
            \State $\widehat\bx \leftarrow \bx$
        \EndIf
    \EndFor
    \State return $\widehat\bx$
    \end{algorithmic}
    \end{algorithm}
    
    We first prove correctness. Let $\bx^*$ be an optimal lottery and $N^* := \{ i\in N : \bx^* \in A_i \}$. By Lemma~\ref{lem:viol-vertex-opt}, some extreme point $\bar\bx$ of $P(N^*)$ is also optimal, and by Lemma~\ref{lem:viol-active-set}, $\bar\bx$ is the unique solution of $\mathcal{L}(J,N_0)$ for some $J\subseteq[m]$ and $N_0\subseteq N^*$ with $|J|+|N_0|=m-1$. When Algorithm~\ref{alg:minviol-full-profile} considers this pair, it finds $\bar\bx\in\Delta(S)$, so the returned lottery satisfies $\viol(\widehat\bx)\le\viol(\bar\bx)=\OPT$; that is, it is optimal.
    
    Next, we prove the running time of Algorithm~\ref{alg:minviol-full-profile}.
    For each $t \in \{0,\dots,m-1\}$, the number of pairs $(J,N_0)$ with $|N_0|=t$ and $|J|=m-1-t$ is $\binom{n}{t}\binom{m}{m-1-t}$, so, by Vandermonde's identity, the total number of pairs is $\sum_{t=0}^{m-1} \binom{n}{t}\binom{m}{m-1-t} = \binom{n+m}{m-1}$.
    For each pair, the system $\mathcal{L}(J,N_0)$ has exactly $m$ equations in $m$ unknowns, so testing consistency and uniqueness and solving it can be done in $\mathcal{O}(m^3)$ time by Gaussian elimination. 
    Evaluating $\viol(\bx)$ requires computing $\langle u_i, \bx\rangle$ for all $i \in N$, which takes $\mathcal{O}(nm)$ time. 
    Thus, the total running time is $\mathcal{O}(\binom{n+m}{m-1}(m^3 + nm))$.
    
    Finally, the decision version (is $\OPT$ at most a given number?) is solved by running Algorithm~\ref{alg:minviol-full-profile} and comparing the violation count of the returned lottery with that number.
    Since $\binom{n+m}{m-1} \le (n+m)^{m-1}$, the running time is polynomial for every constant $m$, and the decision problem is in XP
    when parameterized by $m$.

\subsection{Proof of Theorem~\ref{thm:few-alternatives}}\label{app:proof-few}

    The sampling step below is the argument of \citet{Castiglioni2023persuasion_maxFSrelaxation} for their relaxation of \textsc{Max-FS} over the simplex; we include it for completeness.
    Draw $J_1,\dots,J_{s_\eta}$ independently from $\bx$, i.e., $\Pr[J_t=j]=x_j$ for all $j \in [m]$ and $t \in [s_\eta]$, and let $\tilde{\bx} := \frac{1}{s_\eta}\sum_{t=1}^{s_\eta} \be_{J_t}$ be their empirical distribution.
    Then $\tilde{\bx}\in\Delta(S)$, and it uses at most $s_\eta$ alternatives.
    For a fixed agent $i$, the values $u_i(s_{J_1}),\dots,u_i(s_{J_{s_\eta}})$ are independent, lie in $[0,1]$, and have mean $\sum_{j=1}^m x_j u_i(s_j)=U_i(\bx)$, and $U_i(\tilde{\bx})$ is their average. Hoeffding's inequality gives $\Pr[U_i(\tilde{\bx}) < U_i(\bx)-\eta] \le e^{-2s_\eta\eta^2}$.
    By a union bound and the definition of $s_\eta$, $\Pr[\exists i \in N : U_i(\tilde \bx) < U_i(\bx) - \eta] \le ne^{-2s_\eta\eta^2} \le ne^{-\log(2n)} = 1/2$.
    Hence some realization of $\tilde \bx$ satisfies $U_i(\tilde \bx) \ge U_i(\bx) - \eta$ for all $i \in N$; fix it.

    We now bound the number of violations of $\tilde\bx$. If agent $i$ rejects $\tilde\bx$, then $U_i(\bx)<\tau_i+\eta$, since otherwise $U_i(\tilde \bx) \ge U_i(\bx) - \eta \ge \tau_i$; that is, agent $i$ either rejects $\bx$ or belongs to $B_\eta(\bx)$. Therefore $\viol(\tilde \bx) \le \viol(\bx) + |B_\eta(\bx)|$.
    If $\bx^*$ is optimal and $B_\eta(\bx^*) = \varnothing$, applying this with $\bx=\bx^*$ gives $\viol(\tilde \bx) \le \OPT$, so $\tilde\bx$ is an optimal lottery that uses at most $s_\eta$ alternatives.

\subsection{Proof of Proposition~\ref{prop:bounded-search}}\label{app:proof-bounded}

    For $T \subseteq [m]$, recall from Lemma~\ref{lem:normalized-representative} that $\Delta(T)$ is the set of lotteries that use only alternatives in $T$.
    Every lottery that uses at most $s$ alternatives lies in $\Delta(T)$ for $T = \{j \in [m] : x_j > 0\}$, so it suffices to minimize $\viol$ over $\Delta(T)$ for each $T$ with $1\le t:=|T|\le\min\{s,m\}$ and keep the best solution.

    Fix such a set $T$. If $t=1$, then $\Delta(T)$ consists of a single pure lottery, whose violation count can be computed in $\mathcal{O}(n)$ time. Now suppose $t \ge 2$. On $\Delta(T)$, agent $i$ accepts $\bx$ if and only if $\sum_{j \in T} u_i(s_j)\,x_j \ge \tau_i$, so minimizing $\viol$ over $\Delta(T)$ is the same problem with the $t$ alternatives in $T$. Some agents may accept or reject every lottery in $\Delta(T)$, but the proof of Theorem~\ref{thm:minviol_XP} uses only the simplex constraints, compactness, and the linear acceptance inequalities, so it applies with $m$ replaced by $t$. Hence an optimizer over $\Delta(T)$ can be computed in time
    $\mathcal{O}(\binom{n+t}{t-1}(t^3+nt))$, which also covers the case $t=1$.
    There are $\binom{m}{t}$ sets $T$ of size $t$, and for each $t \le s$ we have $\binom{m}{t}\le m^t$, $\binom{n+t}{t-1}\le (n+s)^{t-1}$, and $t^3+nt\le 2(n+s)^3$; hence all sets of size $t$ together take time $\mathcal{O}(m^t (n+s)^{t+2})$. Summing over the at most $s$ values of $t$ gives total time $\mathcal{O}(s\,m^s (n+s)^{s+2})$, which is $(m(n+s))^{\mathcal{O}(s)}$. The best lottery found over all such $T$ minimizes $\viol$ among all lotteries that use at most $s$ alternatives.

\subsection{Proof of Corollary~\ref{cor:few-alternatives-search}}\label{app:proof-cor}
    The running time is that of Proposition~\ref{prop:bounded-search} with $s=s_\eta$. Let $\widehat\bx$ be the returned lottery and let $\bx$ be any lottery. By Theorem~\ref{thm:few-alternatives}, some lottery $\tilde{\bx}$ that uses at most $s_\eta$ alternatives satisfies $\viol(\tilde{\bx}) \le \viol(\bx) + |B_\eta(\bx)|$. Since $\widehat\bx$ minimizes $\viol$ among all such lotteries, $\viol(\widehat\bx)\le\viol(\tilde\bx)\le\viol(\bx) + |B_\eta(\bx)|$.
    Taking $\bx$ to be an optimal lottery $\bx^*$ gives $\viol(\widehat\bx)\le\OPT+|B_\eta(\bx^*)|$, which equals $\OPT$ if $B_\eta(\bx^*)=\varnothing$.

\section{Omitted Material from Section~\ref{sec:parameterized}}\label{app:full-info}
This appendix proves the algorithm for $m=2$ mentioned in Section~\ref{sec:parameterized}, and Theorems~\ref{thm:fpt} and~\ref{thm:min-viol-w2-hard-k}.

\subsection{Exact optimization on a line segment: $m=2$}\label{app:full-info-m2}

\begin{proposition} \label{prop:viol-m2}
    When $m=2$, there is a deterministic algorithm that, given all normalized pairs
    $(u_i,\tau_i)_{i\in N}$, returns a lottery minimizing $\viol$ in time $\mathcal{O}(n\log n)$.
\end{proposition}
\begin{proof}
    Since $m=2$, every lottery can be written uniquely as $\bx_\lambda := (\lambda,1-\lambda)$ with $\lambda\in[0,1]$. With two alternatives, the normalization of Section~\ref{sec:prelims} gives $u_i=(1,0)$ or $u_i=(0,1)$ for every agent $i$. Hence the set $\Lambda_i := \{\lambda\in[0,1] : \bx_\lambda \in A_i\}$ is the interval $[\ell_i,r_i]$, where $(\ell_i,r_i)=(\tau_i,1)$ if $u_i=(1,0)$ and $(\ell_i,r_i)=(0,1-\tau_i)$ if $u_i=(0,1)$. Since $n-\viol(\bx_\lambda)=|\{i\in N : \lambda\in\Lambda_i\}|$ for every $\lambda\in[0,1]$, it suffices to find a point of $[0,1]$ contained in the largest number of these intervals, which Algorithm~\ref{alg:m=2_minviol} does by sorting the endpoints.

    \begin{algorithm}[t]
    
    \caption{Returns a lottery with the minimum number of violations when $m=2$.}
    \textbf{Input:} The normalized pairs $(u_i,\tau_i)_{i\in N}$, with $S = \{s_1,s_2\}$. \\
    \textbf{Output:} A lottery $\bx^*$ minimizing $\viol$. \par
    \begin{algorithmic}[1]
    \For{$i=1,\dots,n$}
        \If{$u_i=(1,0)$}
            \State Set $(\ell_i,r_i)\gets (\tau_i,1)$
        \Else
            \State Set $(\ell_i,r_i)\gets (0,1-\tau_i)$
        \EndIf
    \EndFor
    \State Form the list of tagged endpoints $(\ell_i,\mathtt{L})$ and $(r_i,\mathtt{R})$ for every $i\in N$, together with $(0,\mathtt{sent})$ and $(1,\mathtt{sent})$
    \State Sort this list by coordinate and scan it once to obtain the distinct coordinates
    $\lambda_1<\dots<\lambda_d$ and, for each $j\in[d]$, the counts
    $L_j:=|\{i\in N:\ell_i=\lambda_j\}|$ and $R_j:=|\{i\in N:r_i=\lambda_j\}|$
    \State $c\gets 0$, $c_{\max}\gets -1$, $\lambda^*\gets 0$
    \For{$j=1,\dots,d$}
        \State $c\gets c+L_j$
        \If{$c>c_{\max}$}
            \State $c_{\max}\gets c$
            \State $\lambda^*\gets \lambda_j$
        \EndIf
        \State $c\gets c-R_j$
    \EndFor
    \State \textbf{return} $\bx^*:= (\lambda^*, 1-\lambda^*)$
    \end{algorithmic}
    \label{alg:m=2_minviol}
    \end{algorithm}
    
    We now prove correctness. Let $\lambda_1<\dots<\lambda_d$ and $L_j,R_j$ be the quantities produced by Algorithm~\ref{alg:m=2_minviol}. 
    We claim that, for each $j\in[d]$, immediately before the update $c\leftarrow c+L_j$, the value of $c$ is $|\{i\in N : \ell_i < \lambda_j \le r_i\}|$.
    This is true for $j=1$ because $\lambda_1=0$ (by the sentinels), initially $c=0$, and no interval satisfies $\ell_i<0$. If the claim holds at the start of iteration $j$, then after the update $c\leftarrow c+L_j$ we have $c = |\{i\in N : \ell_i \le \lambda_j \le r_i\}| = |\{i\in N : \lambda_j\in\Lambda_i\}|$.
    Thus, the algorithm evaluates the overlap number at $\lambda_j$. After this evaluation, the update $c\leftarrow c-R_j$ removes exactly the intervals with $r_i=\lambda_j$, so the resulting value of $c$ is $|\{i\in N : \ell_i \le \lambda_j < r_i\}|$.
    For $j<d$, since there is no endpoint in $(\lambda_j,\lambda_{j+1})$, this equals $|\{i\in N : \ell_i < \lambda_{j+1} \le r_i\}|$, which is precisely the invariant at the start of iteration $j+1$.
    
    Now, let $\lambda\in[0,1]$. 
    If $\lambda=\lambda_j$ for some $j\in[d]$, its overlap number is explicitly evaluated by the algorithm. 
    Otherwise, $\lambda\in(\lambda_j,\lambda_{j+1})$ for some consecutive coordinates $\lambda_j<\lambda_{j+1}$. 
    Since no interval endpoint lies in $(\lambda_j,\lambda_{j+1})$, the set $\{i\in N : \lambda\in\Lambda_i\}$ is independent of the choice of $\lambda$ inside this open interval. 
    Moreover, if $\lambda\in\Lambda_i=[\ell_i,r_i]$, then $\ell_i\notin(\lambda_j,\lambda_{j+1})$ and $r_i\notin(\lambda_j,\lambda_{j+1})$, while $\ell_i\le \lambda\le r_i$. Hence $\ell_i\le \lambda_j$ and $r_i\ge \lambda_{j+1}$, so $\lambda_j\in\Lambda_i$. 
    Therefore, $|\{i\in N : \lambda\in\Lambda_i\}| \le |\{i\in N : \lambda_j\in\Lambda_i\}|$.
    So a maximum overlap point can always be chosen among $\lambda_1,\dots,\lambda_d$, and the algorithm returns some $\lambda^*\in[0,1]$ maximizing $|\{i\in N : \lambda\in\Lambda_i\}|$.
    
    Using $n-\viol(\bx_\lambda)=|\{i\in N:\lambda\in\Lambda_i\}|$, we conclude that $\bx^* := \bx_{\lambda^*}\in \arg\min_{\bx\in\Delta(S)} \viol(\bx)$.
    Finally, computing all values $\ell_i,r_i$ takes $\mathcal{O}(n)$ time, sorting the $2n+2$ tagged endpoints takes $\mathcal{O}(n\log n)$ time, the scan that forms $(\lambda_j,L_j,R_j)_{j=1}^d$ is linear, and the
    final sweep is linear. 
    Therefore, the total running time is $\mathcal{O}(n\log n)$.
\end{proof}

\subsection{Proof of Theorem~\ref{thm:fpt}}\label{app:full-info-fpt}
\begin{proof}
    We first show that an infeasible family of acceptable sets contains an infeasible subfamily of at most $m$ sets, and that such a subfamily can be found with few LP feasibility checks. We then show that branching on which agent of this subfamily to drop decides the problem. Throughout, we use the convention $\bigcap_{i\in\varnothing} A_i=\Delta(S)$.

    \begin{lemma} \label{lem:helly}
        Let $N'\subseteq N$. If $\bigcap_{i\in N'} A_i=\varnothing$, then there exists $W\subseteq N'$ with $|W|\le m$ such that $\bigcap_{i\in W} A_i=\varnothing$.
    \end{lemma}
    \begin{proof}
        Each $A_i$ is a convex subset of the affine hull of $\Delta(S)$, which has dimension $m-1$. By Helly's theorem \citep{Helly1923} in dimension $m-1$, if a finite family of convex subsets of that affine hull has empty intersection, then some subfamily of size at most $(m-1)+1=m$ already has empty intersection. Applying this to $\{A_i\}_{i\in N'}$ proves the claim.
    \end{proof}
    \begin{lemma} \label{lem:helly-compute}
        There is a deterministic procedure that, given $N'\subseteq N$ with $\bigcap_{i\in N'}A_i=\varnothing$,
        returns $W\subseteq N'$ such that $|W|\le m$ and $\bigcap_{i\in W}A_i=\varnothing$ using at most $|N'|$
        feasibility checks of linear programs over $\Delta(S)$.
    \end{lemma}
    
    \begin{proof}
        Initialize $W \leftarrow N'$, and fix an arbitrary ordering
        $i_1,\dots,i_{|N'|}$ of the elements of $N'$.
        For $t = 1, \dots, |N'|$, if $\bigcap_{\ell\in W\setminus\{i_t\}} A_\ell = \varnothing$, then set $W \leftarrow W\setminus\{i_t\}$.
        
        This procedure performs at most $|N'|$ feasibility checks, one for each $i_t$.
        Moreover, throughout the procedure we maintain $\bigcap_{\ell \in W} A_\ell = \varnothing$, because we delete $i_t$ only when the remaining family is still infeasible.
        
        We claim that the final set $W$ is inclusion-minimal infeasible.
        Indeed, fix any $i_t$ that remains in the final set $W$.
        When $i_t$ was processed, it was not deleted, so at that moment the current set, say $W'$, satisfied $\bigcap_{\ell \in W' \setminus \{i_t\}} A_\ell \neq \varnothing$.
        Afterwards the algorithm only deletes more agents, so $W \setminus \{i_t\} \subseteq W' \setminus \{i_t\}$.
        Feasibility is preserved when constraints are removed, and thus $\bigcap_{\ell \in W \setminus \{i_t\}} A_\ell \neq \varnothing$.
        Therefore, $W$ is inclusion-minimal infeasible.
        
        By Lemma~\ref{lem:helly}, every infeasible family contains an infeasible subfamily of size at most $m$.
        If $|W| > m$, then applying Lemma~\ref{lem:helly} to the family $\{A_i\}_{i \in W}$ would give us a proper infeasible subfamily of $W$, contradicting the inclusion-minimality of $W$.
        Therefore $|W| \le m$.
    \end{proof}

    We run Algorithm~\ref{alg:hellybranch} on input $(N,v)$.

    \begin{algorithm}[t]
    \caption{\texttt{HellyBranch}$(N',v)$ (deletes up to $v$ agents from $N'$)}
    \label{alg:hellybranch}
    \textbf{Input:} Index set $N'\subseteq N$, integer $v\ge 0$. \\
    \textbf{Output:} Either \textsc{NO}, or a lottery $\bx \in \Delta(S)$ such that there exists $V \subseteq N'$ with $|V| \le v$ and $\bx \in \bigcap_{i \in N' \setminus V} A_i$.
    \begin{algorithmic}[1]
    \If{$\bigcap_{i\in N'} A_i \neq \varnothing$}
    \State return any $\bx\in\bigcap_{i\in N'} A_i$ (found by LP feasibility)
    \EndIf
    \If{$v=0$}
    \State return \textsc{NO}
    \EndIf
    \State Compute $W\subseteq N'$ with $|W|\le m$ and $\bigcap_{i\in W} A_i = \varnothing$ (Lemma~\ref{lem:helly-compute})
    \For{each $i\in W$}
    \State $\bx \leftarrow \texttt{HellyBranch}(N'\setminus\{i\},v-1)$
    \If{$\bx\neq \textsc{NO}$}
    \State return $\bx$
    \EndIf
    \EndFor
    \State return \textsc{NO}
    \end{algorithmic}
    \end{algorithm}

    We prove the following stronger recursive claim: for every $N' \subseteq N$ and every
    integer $r \ge 0$, the call \texttt{HellyBranch}$(N',r)$ returns \textsc{NO} if and only if there is no set
    $V \subseteq N'$ with $|V| \le r$ such that $\bigcap_{i \in N' \setminus V} A_i \neq \varnothing$.
    Moreover, whenever \texttt{HellyBranch}$(N',r)$ returns a lottery $\bx$, there exists such a set
    $V \subseteq N'$ with $\bx \in \bigcap_{i \in N' \setminus V} A_i.$
    We prove the claim by induction on $r$.
    
    If $r=0$, then the algorithm returns a lottery exactly when $\bigcap_{i \in N'} A_i \neq \varnothing$ (Line~1), and otherwise returns \textsc{NO} (Lines~4--5). 
    This is precisely the claim with $V = \varnothing$.
    
    Now let $r \ge 1$ and assume the claim holds for $r-1$.
    If $\bigcap_{i \in N'} A_i \neq \varnothing$,
    then Line~2 returns some $\bx \in \bigcap_{i \in N'} A_i$,
    so the claim holds with $V=\varnothing$.
    
    Suppose instead that $\bigcap_{i \in N'} A_i = \varnothing$.
    By Lemma~\ref{lem:helly-compute}, the algorithm computes a set $W \subseteq N'$ such that $|W| \le m$ and $\bigcap_{i \in W} A_i = \varnothing$.
    
    If, for some $i \in W$, the recursive call \texttt{HellyBranch}$(N' \setminus \{i\}, r-1)$
    returns a lottery $\bx$, then by the induction hypothesis there exists
    $V' \subseteq N' \setminus \{i\}$ with $|V'| \le r-1$ such that $\bx \in \bigcap_{t \in N' \setminus (V' \cup \{i\})} A_t$.
    Thus, the current call returns a valid lottery with deleted set
    $V := V' \cup \{i\}$, and $|V| \le r$.
    
    Conversely, suppose there exists $V^* \subseteq N'$ with $|V^*| \le r$ such that $\bigcap_{i \in N' \setminus V^*} A_i \neq \varnothing$.
    Since $\bigcap_{i \in W} A_i = \varnothing$, we must have $V^* \cap W \neq \varnothing$; otherwise $W \subseteq N' \setminus V^*$, which would imply $\bigcap_{i \in N' \setminus V^*} A_i
        \subseteq
        \bigcap_{i \in W} A_i
        =
        \varnothing$,  a contradiction. 
    Fix any $i^* \in V^* \cap W$.
    Then $V^* \setminus \{i^*\} \subseteq N' \setminus \{i^*\}$ and $|V^* \setminus \{i^*\}| \le r-1$, and deleting $V^* \setminus \{i^*\}$ from $N' \setminus \{i^*\}$ leaves exactly $N' \setminus V^*$, whose acceptable sets have a common point.
    By the induction hypothesis, the recursive call on
    $(N' \setminus \{i^*\}, r-1)$ returns a lottery, and since the algorithm tries every
    $i \in W$, the current call also returns a lottery.
    This proves the recursive claim. Applying it to the top-level call $(N,v)$, we obtain the desired correctness.
    
    If there exists a lottery $\bx \in \Delta(S)$ with $\viol(\bx) \leq v$, define $V := \{ i \in N : \langle u_i, \bx \rangle < \tau_i \}$.
    By the definition of $A_i$, we have $|V| = \viol(\bx) \le v$ and $\bx \in \bigcap_{i \in N \setminus V} A_i$.
    Thus, by the recursive claim, $\texttt{HellyBranch}(N,v)$ does not return \textsc{NO}.
    
    Conversely, if $\texttt{HellyBranch}(N,v)$ returns a lottery $\bx$, then by the recursive claim there exists
    $V \subseteq N$ with $|V| \le v$ such that $\bx \in \bigcap_{i \in N \setminus V} A_i$.
    Therefore, every violated agent lies in $V$, and so $\viol(\bx) \leq |V| \le v$.
    Thus Algorithm~\ref{alg:hellybranch} returns \textsc{NO} if and only if there is no lottery $\bx \in \Delta(S)$ with $\viol(\bx) \le v$, and whenever it returns a lottery, that lottery has at most $v$ violations.
    
    For the running time, each internal node has at most $m$ children and the recursion depth is at most $v$, so the recursion tree has at most $1+m+\dots+m^v = \mathcal{O}(m^v)$ nodes. 
    At each node, Line~1 solves one LP feasibility problem over $\Delta(S)$; if the current family is infeasible and $r>0$, Lemma~\ref{lem:helly-compute} performs at most $|N'|\le n$ additional LP feasibility checks. 
    Each such LP has $m$ variables and at most $m+|N'|+1\le m+n+1$ linear constraints.
    Thus the algorithm performs $\mathcal{O}(m^v n)$ LP feasibility checks, each of which takes polynomial time, so the problem is fixed-parameter tractable when parameterized by $(v,m)$.
\end{proof}

\subsection{Proof of Theorem~\ref{thm:min-viol-w2-hard-k}}\label{app:full-info-w2}
\begin{proof}
    We give a parameterized reduction from \textsc{Hitting Set}. Let
    $(q,\mathcal{F},v)$ be an instance, where $\mathcal{F}=\{F_1,\dots,F_t\}$ is a family of subsets of $[q]$, and the question is whether there exists a set $C\subseteq [q]$ with $|C|\le v$ such that $C\cap F_\ell\neq \varnothing$ for every
    $\ell\in\{1,\dots,t\}$.
    
    We may assume that $q\ge 1$, that every $F_\ell$ is nonempty, and that $v\le q$: if some $F_\ell$ is empty, the instance is a no-instance, and otherwise, if $q=0$ or $v>q$, it is a yes-instance (take $C=[q]$); in these cases we output a fixed no-instance or yes-instance of \textsc{Min-Viol} with parameter $0$.
    
    From $(q,\mathcal{F},v)$, construct an instance of \textsc{Min-Viol} as follows. 
    Let the set of alternatives be $S=\{s_1,\dots,s_q,s_{q+1}\}$, where $s_{q+1}$ is an extra alternative.
    
    For each $j\in [q]$, create an element agent $a_j$ with threshold $\tau_{a_j}=1/q$ and utility vector $u_{a_j}(s_j)=1$, $u_{a_j}(s_\ell)=0$ for every $\ell \neq j$.
    For any lottery $\bx=(x_1,\dots,x_{q+1}) \in \Delta(S)$, we have $U_{a_j}(\bx)=x_j$.
    Thus, agent $a_j$ accepts $\bx$ if and only if $x_j \ge 1/q$.
    
    Next, for each set $F_\ell\in\mathcal{F}$ and each copy index $r\in\{1,\dots,v+1\}$, create a set agent $b_{\ell,r}$ with threshold $\tau_{b_{\ell,r}} = 1-(|F_\ell|-1)/q$ and utilities $u_{b_{\ell,r}}(s_j)=0$ for $j\in F_\ell$ and $u_{b_{\ell,r}}(s_j)=1$ for $j\in [q+1]\setminus F_\ell$.
    In particular, $u_{b_{\ell,r}}(s_{q+1})=1$. Therefore, for every
    $\bx\in\Delta(S)$, $U_{b_{\ell,r}}(\bx) = \sum_{j\notin F_\ell} x_j   = 1-\sum_{j\in F_\ell} x_j$.
    Thus $b_{\ell,r}$ accepts $\bx$ if and only if $1-\sum_{j\in F_\ell} x_j \ge 1-(|F_\ell|-1)/q$,
    that is, if and only if $\sum_{j\in F_\ell} x_j \le (|F_\ell|-1)/q$.
    
    All utilities in the construction are binary, and all pairs are normalized: every utility vector takes both values $0$ and $1$, and every threshold lies in $(0,1]$. When the threshold is $1$, the utility vector is precisely the indicator vector specified in case~(ii) of Lemma~\ref{lem:normalized-representative}. Moreover, the constructed profile satisfies the nondegeneracy assumption: each element agent $a_j$ accepts $\be_j$ and rejects $\be_{q+1}$; and, since $F_\ell \neq \varnothing$ implies $\tau_{b_{\ell,r}} = (q-|F_\ell|+1)/q \in (0,1]$, each set agent $b_{\ell,r}$ accepts $\be_{q+1}$ and rejects every $\be_j$ with $j \in F_\ell$. 
    
    We now prove correctness.    
    Assume first that the \textsc{Hitting Set} instance is a yes-instance, and let $C\subseteq [q]$ be a hitting set with $|C|\le v$. Define a lottery $\bx^C \in \Delta(S)$ by $x^C_j = 0$ for $j \in C$, $x^C_j = 1/q$ for $j \in [q] \setminus C$, and $x^C_{q+1} = 1 - \sum_{j=1}^q x^C_j = |C|/q$.
    For every $j \in [q]$, the element agent $a_j$ is violated if and only if $j \in C$, since $U_{a_j}(\bx^C)=x^C_j$ and $\tau_{a_j}=1/q$.
    Thus, exactly $|C|$ element agents are violated.
    
    Now fix any $\ell \in \{1,\dots,t\}$. Since $C$ is a hitting set, $C \cap F_\ell \neq \varnothing$, so at most $|F_\ell|-1$ elements of $F_\ell$ belong to $[q] \setminus C$. 
    Since $x_j^C = 1/q$ for $j \in [q] \setminus C$ and $x_j^C = 0$ for $j \in C$, it follows that $\sum_{j\in F_\ell} x_j^C \le (|F_\ell|-1)/q$.
    Hence every copy $b_{\ell,r}$ accepts $\bx^C$. Thus no set agent is violated.
    Therefore, $\viol(\bx^C)=|C|\le v$, so the constructed \textsc{Min-Viol} instance is a yes-instance.
    
    Conversely, suppose there exists a lottery $\bx\in\Delta(S)$ such that $\viol(\bx)\le v$.
    We first show that $\bx$ violates no set agent. 
    Indeed, if $\bx$ violated some $b_{\ell,r}$, then it would violate all $v+1$ identical copies $b_{\ell,1},\dots,b_{\ell,v+1}$, which would imply $\viol(\bx)\ge v+1$, contradicting $\viol(\bx) \le v$.
    
    Thus, every set agent accepts $\bx$, and so for every $\ell\in\{1,\dots,t\}$ we have $\sum_{j\in F_\ell} x_j \le (|F_\ell|-1)/q$.
    Define $C(\bx) := \{ j \in [q] : x_j < 1/q \}$.
    We claim that $C(\bx)$ is a hitting set. 
    Suppose not. 
    Then there exists some $\ell \in \{1,\dots,t\}$ such that $C(\bx) \cap F_\ell = \varnothing$.
    Thus, $x_j \ge 1/q$ for every $j \in F_\ell$, and therefore $\sum_{j \in F_\ell} x_j \ge |F_\ell|/q > (|F_\ell|-1)/q$, contradicting the acceptance of the set agents associated with $F_\ell$. 
    Thus $C(\bx)$ intersects every set $F_\ell$, so it is a hitting set.
    
    Finally, an element agent $a_j$ is violated if and only if $x_j < 1/q$, that is, if and only if $j \in C(\bx)$. 
    Since no set agent is violated, it
    follows that $|C(\bx)|=\viol(\bx)\le v$.
    So $C(\bx)$ is a hitting set of size at most $v$.
    
    We have shown that in the case $q\ge 1$, every $F_\ell\neq\varnothing$, and $v\le q$, the instance $(q,\mathcal{F},v)$ is a yes-instance of \textsc{Hitting Set} if and only if the constructed \textsc{Min-Viol} instance is a yes-instance. In this branch the construction uses $q+1$ alternatives and $q+t(v+1)$ agents, and since $v\le q$, its size is polynomial in the size of $(q,\mathcal{F},v)$. The output parameter is $v$, or $0$ in the trivial cases above. Therefore this is a parameterized reduction from \textsc{Hitting Set} parameterized by $v$ to the decision version of \textsc{Min-Viol}. Since \textsc{Hitting Set} is W[2]-hard parameterized by $v$, the latter problem is also W[2]-hard when parameterized by $v$  alone.
    
    The same hardness immediately carries over to the search version: any
    algorithm that returns a lottery $\bx$ with $\viol(\bx)\leq v$ whenever one exists also decides the decision problem.
\end{proof}

\section{Omitted Proofs from Section~\ref{sec:statistical}}

Throughout this appendix, $i_1,\dots,i_k$ are sampled independently from $\mathrm{Unif}(N)$. For a lottery $\bx$, let $R_\bx := \{i \in N : \bx \notin A_i\}$ be the set of agents who reject $\bx$, let $r(\bx) := |R_\bx|/n = \viol(\bx)/n$, and let $\widehat r(\bx) := \frac1k \sum_{t=1}^k \mathbf{1}\{i_t\in R_\bx\}$ be the fraction of sampled agents who reject $\bx$, counted with multiplicity. The lotteries $\widehat\bx$ in Theorems~\ref{thm:budget-additive} and~\ref{thm:budget-relative} are exactly the minimizers of $\widehat r$. Both $r$ and $\widehat r$ take finitely many values, so all minima below are attained.

\subsection{Proof of Theorem~\ref{thm:budget-additive}}\label{app:proof-additive}
    Since $\sum_{j=1}^m x_j = 1$ for every lottery, we have $i \in R_\bx$ if and only if $\langle u_i, \bx\rangle < \tau_i$, that is, if and only if $\sum_{j=1}^m (u_i(s_j)-\tau_i)x_j < 0$.
    If we identify agent $i$ with the point $(u_i(s_1)-\tau_i,\dots,u_i(s_m)-\tau_i)\in\mathbb{R}^m$, then $R_\bx$ is the set of agents whose points lie in the open halfspace $\{\mathbf{y}\in\mathbb{R}^m : \langle \mathbf{y},\bx\rangle < 0\}$. Hence the family $\{R_\bx:\bx\in\Delta(S)\}$ is induced on $N$ by a subfamily of the open halfspaces of $\mathbb{R}^m$ whose boundary passes through the origin, and its VC dimension is at most $m$.
    Since $r(\bx)=\Pr_{i\sim \mathrm{Unif}(N)}[i\in R_\bx]$, the VC inequality \citep{Vapnik1998,AnthonyBartlett1999} together with Sauer's lemma (using $k\ge m$) gives, with probability at least $1-\delta$, $\sup_{\bx\in\Delta(S)} |\widehat r(\bx)-r(\bx)| \le 2\sqrt{2(m\log(2ek/m)+\log(4/\delta))/k} \le \alpha$, where $\alpha:=4\sqrt{(m\log(2ek/m)+\log(8/\delta))/k}$.
    On this event, if $\bx^*$ minimizes $r$ and $\widehat\bx$ minimizes $\widehat r$, then $r(\widehat \bx) \le \widehat r(\widehat \bx) + \alpha \le \widehat r(\bx^*) + \alpha \le r(\bx^*) + 2\alpha$.
    Multiplying by $n$ gives $\viol(\widehat\bx)\le\OPT+2\alpha n$, which is the claim.

\subsection{Proof of Theorem~\ref{thm:budget-relative}}\label{app:relative-sampling}
    The proof follows that of Theorem~\ref{thm:budget-additive}, with the VC inequality replaced by a bound on the relative error, which is much smaller for lotteries rejected by few agents. By the proof of Theorem~\ref{thm:budget-additive}, the family $\{R_\bx:\bx\in\Delta(S)\}$ has VC dimension at most $m$.
    Let $\beta:=\zeta/20$. By the relative-error bound of \citet[Theorem~5]{LiLongSrinivasan2001}, applied with parameter $\rho$ and accuracy $\beta$, there is an absolute constant $C$ such that, if $k\ge C(m\log(1/\rho)+\log(1/\delta))/(\zeta^2\rho)$, then with probability at least $1-\delta$, $|\widehat r(\bx)-r(\bx)|\le\beta(\widehat r(\bx)+r(\bx)+\rho)$ for all $\bx\in\Delta(S)$.
    Fix this event. Rearranging the two directions of this inequality gives $(1-\beta)\,r(\bx)\le(1+\beta)\,\widehat r(\bx)+\beta\rho$ and $(1-\beta)\,\widehat r(\bx)\le(1+\beta)\,r(\bx)+\beta\rho$ for every lottery $\bx$.
    Let $\bx^*$ minimize $r$, so that $r(\bx^*)=\OPT/n$. Since $\widehat\bx$ minimizes $\widehat r$,
    \[
    (1-\beta)\,r(\widehat\bx)
    \le(1+\beta)\,\widehat r(\widehat\bx)+\beta\rho
    \le(1+\beta)\,\widehat r(\bx^*)+\beta\rho
    \le\frac{(1+\beta)^2}{1-\beta}\,r(\bx^*)+\frac{1+\beta}{1-\beta}\,\beta\rho+\beta\rho,
    \]
    and dividing by $1-\beta$ gives $r(\widehat\bx)\le((1+\beta)/(1-\beta))^2\, r(\bx^*)+2\beta\rho/(1-\beta)^2$.
    For $0<\beta\le 1/20$, we have $((1+\beta)/(1-\beta))^2\le 1+10\beta$ (the left-hand side is convex in $\beta$, equals $1$ at $\beta=0$, and is below $1.23$ at $\beta=1/20$) and, since $(1-\beta)^2\ge 9/10$, $2\beta/(1-\beta)^2\le 10\beta$. Since $10\beta=\zeta/2$, multiplying by $n$ gives $\viol(\widehat\bx)\le(1+\zeta/2)\OPT+(\zeta/2)\rho n$.
    If $\OPT\ge\rho n$, then $(\zeta/2)\rho n\le(\zeta/2)\OPT$, so $\viol(\widehat\bx)\le(1+\zeta)\OPT$.

\subsection{Proof of Theorem~\ref{thm:shortfall-main}}\label{app:proof-lp}

    We show that the LP penalty is an upper bound on the violation indicator (Lemma~\ref{lem:shortfall-domination}), that its sample average is uniformly close to its population average (Lemma~\ref{lem:shortfall-uniform}), and that its population average for any lottery $\bx$ is at most $((1+1/\eta)\viol(\bx)+|B_\eta(\bx)|)/n$.
    For an agent $i$ and a lottery $\bx$, let $\ell(i,\bx):=\max\{0,(\tau_i+\eta-U_i(\bx))/\eta\}$ be the penalty of Section~\ref{subsec:shortfall-lp}, and let $\widehat L(\bx):=\frac{1}{k}\sum_{t=1}^k \ell(i_t,\bx)$ and $L(\bx):=\frac{1}{n}\sum_{i\in N}\ell(i,\bx)$ be its sample and population averages.

    \begin{lemma} \label{lem:shortfall-domination}
        For every agent $i$ and lottery $\bx$, $\mathbf{1}\{U_i(\bx)<\tau_i\}\le \ell(i,\bx)$.
        Moreover, the optimal value of \eqref{eq:shortfall-lp} is $\min_{\bx\in\Delta(S)} \widehat L(\bx)$, and the $\bx$-part of every optimal solution of \eqref{eq:shortfall-lp} minimizes $\widehat L$.
    \end{lemma}

    \begin{proof}
        If $U_i(\bx)\ge \tau_i$, the indicator is $0\le \ell(i,\bx)$. If $U_i(\bx)<\tau_i$, then $\tau_i+\eta-U_i(\bx)>\eta$, so $\ell(i,\bx)>1$.
        For a fixed $\bx$, the smallest feasible value of $z_t$ in \eqref{eq:shortfall-lp} is $\ell(i_t,\bx)$, so the smallest objective value with this $\bx$ is $\widehat L(\bx)$. Hence the optimal value of \eqref{eq:shortfall-lp} is $\min_{\bx}\widehat L(\bx)$, and if $(\bx,z_1,\dots,z_k)$ is optimal, then $\widehat L(\bx)\le\frac1k\sum_t z_t=\min_{\bx'}\widehat L(\bx')$.
    \end{proof}

    \begin{lemma} \label{lem:shortfall-uniform}
        There is an absolute constant $C>0$ such that, for every integer $k\ge 1$, $\delta\in(0,1)$, and $\eta\in(0,1]$, with probability at least $1-\delta$, $\sup_{\bx\in\Delta(S)} |\widehat L(\bx)-L(\bx)| \le (C/\eta)\sqrt{(\log m+\log(1/\delta))/k}$.
    \end{lemma}

    \begin{proof}
        We use the empirical Rademacher complexity of a class $\mathcal H$ of functions on $N$, namely $\widehat{\mathfrak R}(\mathcal H):=\mathbb{E}_\sigma[\sup_{g\in\mathcal H}\frac1k\sum_{t=1}^k\sigma_t\, g(i_t)]$, where $\sigma_1,\dots,\sigma_k$ are independent and uniform on $\{-1,1\}$, together with the standard facts about it in \citet[Chapter~26]{ShalevShwartzBenDavid2014}.

        For each lottery $\bx$, define $f_\bx(i):=U_i(\bx)-\tau_i$, and for each alternative $s_j$, define $h_j(i):=u_i(s_j)-\tau_i$. Since $x_j\ge 0$ and $\sum_{j=1}^m x_j=1$, we have $f_\bx=\sum_{j=1}^m x_j h_j$, so the class $\mathcal F:=\{f_\bx:\bx\in\Delta(S)\}$ is the convex hull of $h_1,\dots,h_m$. Taking the convex hull does not change the Rademacher complexity, and since each $h_j$ takes values in $[-1,1]$, Massart's lemma gives $\widehat{\mathfrak R}(\mathcal F)=\widehat{\mathfrak R}(\{h_1,\dots,h_m\})\le 2\sqrt{(2\log m)/k}$.

        Let $\psi(z):=\max\{0,1-z/\eta\}-1$, so that $\ell(i,\bx)=1+\psi(f_\bx(i))$, and let $\mathcal G:=\{\psi\circ f_\bx:\bx\in\Delta(S)\}$. Since $\psi$ is $(1/\eta)$-Lipschitz, the contraction lemma gives $\widehat{\mathfrak R}(\mathcal G)\le\widehat{\mathfrak R}(\mathcal F)/\eta$. Moreover, $f_\bx(i)\in[-1,1]$ because $U_i(\bx)\in[0,1]$ and $\tau_i\in(0,1]$, so every function in $\mathcal G$ takes values in $[-1,1/\eta]\subseteq[-1/\eta,1/\eta]$. The uniform convergence bound in terms of $\widehat{\mathfrak R}$ \citep[Theorem~26.5]{ShalevShwartzBenDavid2014}, applied to $\mathcal G$ and to $-\mathcal G$ (which has the same Rademacher complexity, since $\sigma$ is symmetric), each with $\delta/2$ in place of $\delta$, shows that with probability at least $1-\delta$,
        \[
        \sup_{g\in \mathcal G}
        \left|
        \frac1k\sum_{t=1}^k g(i_t)-\frac1n\sum_{i\in N} g(i)
        \right|
        \le
        2\widehat{\mathfrak R}(\mathcal G)+\frac{4}{\eta}\sqrt{\frac{2\log(8/\delta)}{k}}
        \le \frac{4}{\eta}\sqrt{\frac{2\log m}{k}}+\frac{4}{\eta}\sqrt{\frac{2\log(8/\delta)}{k}}.
        \]
        Since the constant $1$ in $\ell(i,\bx)=1+\psi(f_\bx(i))$ cancels, the left-hand side equals $\sup_{\bx\in\Delta(S)}|\widehat L(\bx)-L(\bx)|$. Finally, $\log(8/\delta)\le 3\log m+\log(1/\delta)$ because $m\ge 2$, which gives the lemma.
    \end{proof}

    Let $\kappa:=(C/\eta)\sqrt{(\log m+\log(1/\delta))/k}$ be the bound of Lemma~\ref{lem:shortfall-uniform}, and fix the event of that lemma. We use three bounds. First, by Lemma~\ref{lem:shortfall-domination}, $\viol(\widehat \bx_\eta)\le nL(\widehat \bx_\eta)$. Second, since $\widehat \bx_\eta$ minimizes $\widehat L$ (Lemma~\ref{lem:shortfall-domination}), every lottery $\bx$ satisfies $L(\widehat \bx_\eta) \le \widehat L(\widehat \bx_\eta)+\kappa \le \widehat L(\bx)+\kappa \le L(\bx)+2\kappa$.
    Third, we bound $nL(\bx)$ for a fixed lottery $\bx$. If $U_i(\bx)\ge \tau_i+\eta$, then $\ell(i,\bx)=0$. If $i\in B_\eta(\bx)$, then $0<\ell(i,\bx)\le 1$. If $U_i(\bx)<\tau_i$, then $\ell(i,\bx)= (\tau_i+\eta-U_i(\bx))/\eta\le (1+\eta)/\eta=1+1/\eta$, because $\tau_i\le 1$ and $U_i(\bx)\ge 0$. Therefore $nL(\bx)\le(1+1/\eta) \viol(\bx)+|B_\eta(\bx)|$.
    Combining the three bounds gives $\viol(\widehat{\bx}_\eta) \le nL(\bx)+2\kappa n \le (1+1/\eta)\viol(\bx)+|B_\eta(\bx)|+2\kappa n$ for every lottery $\bx$, which is the claim of Theorem~\ref{thm:shortfall-main} with constant $2C$.
    If the LP is run on all $n$ agents, each once, then $\widehat L=L$, and the same argument with $\kappa=0$ gives the full-information bound stated in Section~\ref{subsec:shortfall-lp}.

\section{Converting Sample Sizes into Query Budgets} \label{app:budget-implementations}
Suppose each agent can be recovered with at most $T_\eps= \mathcal{O}(m\log(1/\eps))$ queries, as under the finite-grid convention. With a budget of $Q$ queries, sample $k:=\lfloor Q/T_\eps\rfloor$ agents independently from $\mathrm{Unif}(N)$ and recover only the distinct sampled agents. This uses at most $kT_\eps\le Q$ queries, and the sampled agents, with their multiplicities, then define the problem on the sample exactly.

\begin{proposition}
\label{prop:budget-additive}
    Assume the finite-grid convention and that $m$ is a constant. Fix $\delta\in(0,1)$, let $k:=\lfloor Q/T_\eps\rfloor\ge m$, sample and recover agents as above, and minimize the number of rejecting sampled agents exactly with the algorithm of Theorem~\ref{thm:minviol_XP}. This uses at most $Q$ queries, the optimization takes time polynomial in $k$, and with probability at least $1-\delta$ the returned lottery satisfies the bound of Theorem~\ref{thm:budget-additive}. If, in addition, $\delta$, $\rho$, $\zeta$, and $k$ satisfy the conditions of Theorem~\ref{thm:budget-relative} and $\OPT\ge\rho n$, then with probability at least $1-\delta$ the returned lottery has at most $(1+\zeta)\OPT$ violations.
\end{proposition}

\begin{proof}
    The query count was shown above. The sampled agents, counted with multiplicity, form a profile of $k$ agents, and Theorem~\ref{thm:minviol_XP} minimizes the number of violations on it exactly in time polynomial in $k$ for constant $m$; for $m=2$, Proposition~\ref{prop:viol-m2} takes time $\mathcal{O}(k\log k)$. The guarantees then follow from Theorems~\ref{thm:budget-additive} and~\ref{thm:budget-relative}.
\end{proof}

\begin{proposition}
\label{prop:shortfall-budget}
	Assume the finite-grid convention. Fix $\delta\in(0,1)$ and $\eta\in(0,1]$, let $k:=\lfloor Q/T_\eps\rfloor\ge 1$, sample and recover agents as above, and solve the linear program \eqref{eq:shortfall-lp}. This uses at most $Q$ queries, the optimization takes time polynomial in $k$ and $m$, and with probability at least $1-\delta$ the returned lottery satisfies the bound of Theorem~\ref{thm:shortfall-main}.
\end{proposition}
\begin{proof}
The query count was shown above. Once the sampled agents are recovered, \eqref{eq:shortfall-lp} is a linear program with $m+k$ variables and $\mathcal{O}(m+k)$ constraints, which can be solved in polynomial time. The guarantee follows from Theorem~\ref{thm:shortfall-main}.
\end{proof}

\section{Omitted Proofs from Section~\ref{sec:deterministic_worst_case}}

\subsection{Proof of Theorem~\ref{thm:coverage}}\label{app:proof-coverage}

    The idea is to choose utilities that reproduce every answer the algorithm received, while making every unqueried agent accept only a pure lottery other than $\widehat\bx$. Consider the run of $\mathcal{A}$ in which every query $(i,\bx)$ is answered ``accept'' if and only if $x_3=\cdots=x_m=0$, with queried set $M$ and output $\widehat\bx$.
    For each queried agent $i \in M$, set $u_i=(1,1,0,\dots,0)$ and $\tau_i = 1$. Then $A_i = \{\bx \in \Delta(S) : x_1+x_2 \ge 1\} = \{\bx \in \Delta(S) : x_3=\dots=x_m=0\}$, so these agents give exactly the answers of the run above.
    For each unqueried agent $i \in N \setminus M$, set $\tau_i=1$ and $u_i=\mathbf{z}$, where $\mathbf{z}:=\be_1$ if $\widehat{\bx} \neq \be_1$ and $\mathbf{z}:=\be_2$ if $\widehat{\bx} = \be_1$.
    Let $\I:=((u_i,\tau_i))_{i\in N}$. Every unqueried agent accepts exactly one lottery, the pure lottery $\mathbf{z}$, so it rejects $\widehat{\bx}$.
    All agents are nondegenerate: queried agents accept exactly the lotteries $(\lambda,1-\lambda,0,\dots,0)$ with $\lambda \in [0,1]$ and reject $\be_3$, which exists because $m\ge 3$, while each unqueried agent accepts only $\mathbf{z}$. All pairs are indicator vectors of proper subsets of alternatives with threshold $1$, so they are normalized.

    We claim that running $\mathcal{A}$ on $\I$ produces the same query history as the original run. We show by induction on $t$ that the first $t$ query--answer pairs agree, for every $t$ up to the number of queries in the original run. The case $t=0$ is trivial. If the first $t$ pairs agree, then, because $\mathcal{A}$ is deterministic, its $(t+1)$-st query is the same in both runs, so the queried agent belongs to $M$. For agents in $M$, the answers under $\I$ agree with the answer rule of the original run, so the $(t+1)$-st answer is also the same. Hence the entire query history is the same, and since $\mathcal{A}$ is deterministic, it returns $\widehat{\bx}$ on $\I$.

    Every agent in $N \setminus M$ rejects $\widehat{\bx}$, so $\viol_\I(\widehat{\bx}) \ge n-|M|$. Every queried agent accepts $\mathbf{z}\in \{\be_1,\be_2\}$, and every unqueried agent accepts $\mathbf{z}$ by construction, so $\viol_\I(\mathbf{z})=0$ and $\OPT(\I) = 0$.

\subsection{Proof of Theorem~\ref{thm:exact_deterministic_benchmark}}\label{app:proof-exact}
    \textbf{Upper bound.} Fix $M\subseteq N$ with $|M|=K$. Consider the deterministic algorithm that queries only agents in $M$, recovers each of them, and returns a lottery $\bx_M$ minimizing the number of violated agents in $M$, chosen by a fixed tie-breaking rule. This algorithm belongs to $\mathcal{C}_K$. For any profile $\I$, let $\bx^*$ be optimal for $\I$. Then $\bx_M$ violates at most as many agents of $M$ as $\bx^*$ does, and at most $n-K$ agents outside $M$, so $\viol_\I(\bx_M)\le \viol_\I(\bx^*) + (n-K)=\OPT(\I)+(n-K)$.
    Hence $\inf_{\mathcal{A}\in\mathcal{C}_K}\sup_\I (\viol_\I(\mathcal{A}(\I))-\OPT(\I))\le n-K$.

    \textbf{Lower bound.} Fix any $\mathcal{A} \in \mathcal{C}_K$ and run it on the profile in which every agent has utility vector $(1,1,0,\dots,0)$ and threshold $1$; its answers are exactly those of the answer rule in Theorem~\ref{thm:coverage}. Let $M$ be the set of agents queried and $\widehat{\bx}$ the output. This profile satisfies the model assumptions, so $|M| \le K$.
    By Theorem~\ref{thm:coverage}, there is a profile $\I$ on which $\mathcal{A}$ returns $\mathcal{A}(\I)=\widehat{\bx}$, every agent in $N\setminus M$ rejects $\widehat{\bx}$, and $\OPT(\I)=0$. Therefore $\viol_\I(\mathcal{A}(\I))-\OPT(\I) = \viol_\I(\widehat{\bx}) \ge n-|M| \ge n-K$.
    The profile $\I$ satisfies the model assumptions: it is nondegenerate by the proof of Theorem~\ref{thm:coverage}, and since its utilities are binary and its thresholds equal $1$, it satisfies the finite-grid convention for every $\eps$. Hence $\sup_\I(\viol_\I(\mathcal{A}(\I))-\OPT(\I))\ge n-K$ for every $\mathcal{A}\in\mathcal{C}_K$, which proves the claim.

\section{Bounds After a Partial Query History}\label{app:partial-elicitation}
This appendix studies what can be said after an arbitrary query history, before all agents have been recovered: we bound the unknown optimum, and the worst-case loss of a chosen lottery, using only the answers observed so far.

\subsection{Bounds from observed queries}\label{app:partial-elicitation-bounds}
Throughout this subsection, no finite-precision assumption is needed.

A partial query history is the set of accept/reject answers observed so far. We
write it as a finite set $h\subseteq N\times\Delta(S)\times\{0,1\}$, where
$(i,\bx,b)\in h$ means that querying agent $i$ on lottery $\bx$ returned answer $b$.
For each agent $i\in N$, let $h_i := \{(\bx,b) : (i,\bx,b)\in h\}$ be the answers observed for agent $i$.
We say that $h'$ \emph{extends} $h$ if $h\subseteq h'$, and that $h$ is \emph{consistent} if there exists at least one profile satisfying our model whose answers reproduce every triple in $h$.

For each agent $i$, let $\Theta_i(h_i)$ be the set of normalized pairs $(u,\tau)$, in the sense of Section~\ref{sec:prelims}, that agree with all observed answers, i.e., $\mathbf{1}\{\langle u,\bx\rangle\ge \tau\}=b$ for every $(\bx,b)\in h_i$. Thus $\Theta_i(h_i)$ contains the pairs that remain possible for agent $i$, and $h$ is consistent if and only if $\Theta_i(h_i)\neq\varnothing$ for every $i\in N$.

A profile $\I=((u_i,\tau_i))_{i\in N}$ is \emph{consistent with $h$} if $(u_i,\tau_i)\in \Theta_i(h_i)$ for every $i\in N$; we use $\viol_\I$ and $\OPT(\I)$ as in Section~\ref{sec:deterministic_worst_case}.
The answers observed so far determine two sets of lotteries for each agent: $\underline{A}_i(h_i) := \{\bx \in \Delta(S) : \langle u, \bx\rangle \ge \tau \text{ for every } (u,\tau) \in \Theta_i(h_i)\}$ and $\overline{A}_i(h_i) := \{\bx \in \Delta(S) : \langle u, \bx\rangle \ge \tau \text{ for some } (u,\tau) \in \Theta_i(h_i)\}$.
The first set contains the lotteries that agent $i$ is guaranteed to accept given the observed answers, and the second contains those that agent $i$ accepts under at least one consistent profile.
If $h$ is consistent, then $\Theta_i(h_i)\neq \varnothing$, so $\underline{A}_i(h_i)\subseteq \overline{A}_i(h_i)$, and if $\Theta_i(h_i)=\{(u_i,\tau_i)\}$, then both sets equal $A_i$.

For any lottery $\bx\in \Delta(S)$, define the lower and upper violation counts
consistent with the observed answers as $\underline{\viol}_h(\bx):=|\{i\in N : \bx\notin \overline{A}_i(h_i)\}|$ and $\overline{\viol}_h(\bx):=|\{i\in N : \bx\notin \underline{A}_i(h_i)\}|$.
The lower count $\underline{\viol}_h(\bx)$ counts agents who must reject
$\bx$ under every profile still consistent with the observed answers, while the
upper count $\overline{\viol}_h(\bx)$ counts agents who could reject $\bx$
under at least one consistent profile.

\begin{theorem}
	For every consistent query history $h$, every lottery $\bx\in\Delta(S)$, and every profile $\I$ consistent with $h$, $\underline{\viol}_h(\bx)\le \viol_\I(\bx)\le \overline{\viol}_h(\bx)$.
	Consequently, if $L(h):=\min_{\bx\in\Delta(S)} \underline{\viol}_h(\bx)$ and $H(h):=\min_{\bx\in\Delta(S)} \overline{\viol}_h(\bx)$, then every profile $\I$ consistent with $h$ satisfies $L(h)\le \OPT(\I)\le H(h)$.
	Moreover, if $\widehat{\bx}_h\in\arg\min_{\bx\in\Delta(S)} \overline{\viol}_h(\bx)$, then $\viol_\I(\widehat{\bx}_h)\le \OPT(\I)+(H(h)-L(h))$.
	In particular, if $H(h)=L(h)$, then $\widehat{\bx}_h$ is optimal for every profile consistent with $h$.
\end{theorem}

\begin{proof}
	Fix $\bx \in \Delta(S)$ and a profile $\I$ consistent with $h$. 
	For each agent $i$, we have $(u_i,\tau_i) \in \Theta_i(h_i)$. Therefore:
	\begin{itemize}
	    \item if $\bx \notin \overline{A}_i(h_i)$, then every pair in $\Theta_i(h_i)$ rejects $\bx$, so agent $i$ rejects $\bx$ in profile $\I$;
	    \item if $\bx \in \underline{A}_i(h_i)$, then every pair in $\Theta_i(h_i)$ accepts $\bx$, so agent $i$ accepts $\bx$ in profile $\I$.
	\end{itemize}
	Summing over agents gives $\underline{\viol}_h(\bx) \le \viol_\I(\bx) \le \overline{\viol}_h(\bx)$.
Minimizing each side over $\bx\in\Delta(S)$ gives $L(h)\le\OPT(\I)\le H(h)$. Finally, $\viol_\I(\widehat{\bx}_h)\le \overline{\viol}_h(\widehat{\bx}_h)=H(h)$, and combining this with $\OPT(\I)\ge L(h)$ gives $\viol_\I(\widehat{\bx}_h) - \OPT(\I)\le H(h)-L(h)$.
\end{proof}

\begin{proposition}
	Let $h'$ extend $h$, and assume that $h'$ is consistent. 
	Then, for every lottery $\bx \in \Delta(S)$, $\underline{\viol}_h(\bx) \le \underline{\viol}_{h'}(\bx) \le \overline{\viol}_{h'}(\bx) \le \overline{\viol}_h(\bx)$.
Consequently, $L(h)\le L(h')\le H(h')\le H(h)$.
In particular, the uncertainty gap $H(h)-L(h)$ cannot increase as more queries are answered.
\end{proposition}

\begin{proof}
If $h'$ extends $h$, then $\Theta_i(h'_i) \subseteq \Theta_i(h_i)$ for every agent $i$. Thus, $\underline{A}_i(h_i) \subseteq \underline{A}_i(h'_i)$ and $\overline{A}_i(h'_i) \subseteq \overline{A}_i(h_i)$.
Therefore, for each fixed lottery $\bx$, the set of agents for whom $\bx$ is not even possibly acceptable can only grow, while the set of agents for whom $\bx$ is guaranteed acceptable can only grow. Equivalently, $\underline{\viol}_h(\bx)$ can only increase and $\overline{\viol}_h(\bx)$ can only decrease. This proves the first inequality, and minimizing over $\bx \in \Delta(S)$ gives the second inequality.
\end{proof}

A useful extreme case is an agent with no observed answers. If $h_i=\varnothing$, then $\Theta_i(h_i)$ is the set of all normalized pairs. Thus, every lottery $\bx\in\Delta(S)$ is accepted by some pair in $\Theta_i(h_i)$ and rejected by some (possibly different) pair in $\Theta_i(h_i)$: to accept $\bx$, choose any $j$ with $x_j>0$ and take $u=\be_j$, $\tau=x_j/2$; to reject $\bx$, choose any $\ell$ with $x_\ell<1$ (which exists because $m\ge 2$) and take $u=\be_\ell$, $\tau=(x_\ell+1)/2$. Both pairs are normalized, since their thresholds lie strictly between $0$ and $1$. Therefore $\overline{A}_i(h_i)=\Delta(S)$ and $\underline{A}_i(h_i)=\varnothing$.
So an unseen agent contributes $0$ to $\underline{\viol}_h(\bx)$ and $1$ to $\overline{\viol}_h(\bx)$ for every lottery $\bx$. 
In particular, if exactly $K$ agents have nonempty histories, then $\overline{\viol}_h(\bx)-\underline{\viol}_h(\bx)\ge n-K$ for every $\bx\in\Delta(S)$, and thus $H(h)-L(h)\ge n-K$. 
This shows, for an arbitrary partial history, why unqueried agents drive the lower bound of Theorem~\ref{thm:coverage}.

The bound $H(h)-L(h)$ for $\widehat{\bx}_h$ can be loose. To state the exact worst-case additive gap of an arbitrary fixed lottery $\bx$ over all profiles consistent with $h$, define, for $\bx, \mathbf{z} \in \Delta(S)$,
\[
D_h(\bx,\mathbf{z}) := \sum_{i \in N}
\max_{(u,\tau) \in \Theta_i(h_i)}
(\mathbf{1}\{\langle u, \bx\rangle < \tau\} - \mathbf{1}\{\langle u, \mathbf{z}\rangle < \tau\}).
\]

\begin{theorem} \label{thm:exact_worst-case_gap}
For every consistent query history $h$ and every lottery $\bx \in \Delta(S)$, the supremum of $\viol_\I(\bx) - \OPT(\I)$ over all profiles $\I$ consistent with $h$ equals $\max_{\mathbf{z} \in \Delta(S)} D_h(\bx,\mathbf{z})$.
Consequently, the smallest worst-case additive gap that any lottery can guarantee after observing $h$ is $\min_{\bx \in \Delta(S)} \max_{\mathbf{z} \in \Delta(S)} D_h(\bx,\mathbf{z})$.
\end{theorem}

\begin{proof}
	Fix $\bx \in \Delta(S)$. For each agent $i$, the expression maximized in the $i$-th summand of $D_h$ takes values in $\{-1,0,1\}$ on the nonempty set $\Theta_i(h_i)$, so its maximum is attained; and since $D_h(\bx,\mathbf{z})$ is an integer in $\{-n,\dots,n\}$, the maximum over $\mathbf{z}\in\Delta(S)$ is attained as well.
	For the upper bound, let $\I$ be any profile consistent with $h$, and let $\mathbf{z}_\I \in \arg\min_{\mathbf{z} \in \Delta(S)} \viol_\I(\mathbf{z})$.
Then $\viol_\I(\bx)-\OPT(\I) = \viol_\I(\bx)-\viol_\I(\mathbf{z}_\I)$, which is the sum over $i \in N$ of $\mathbf{1}\{\langle u_i,\bx\rangle < \tau_i\} - \mathbf{1}\{\langle u_i,\mathbf{z}_\I\rangle < \tau_i\}$.
Since $(u_i,\tau_i) \in \Theta_i(h_i)$ for every $i$, each summand is at most the corresponding term in the definition of $D_h(\bx,\mathbf{z}_\I)$. Thus $\viol_\I(\bx)-\OPT(\I) \le D_h(\bx,\mathbf{z}_\I) \le \max_{ \mathbf{z} \in \Delta(S)} D_h(\bx,\mathbf{z})$, and taking the supremum over all profiles $\I$ consistent with $h$ gives one direction of the claimed equality.

For the reverse inequality, let $\mathbf{z}^* \in \arg\max_{\mathbf{z} \in \Delta(S)} D_h(\bx,\mathbf{z})$.
For each agent $i$, choose $(u_i^*,\tau_i^*)\in\Theta_i(h_i)$ attaining the maximum in the $i$-th summand of $D_h(\bx,\mathbf{z}^*)$. 
Since the consistency constraints are separate across agents, these choices combine into a profile $\I^* := ((u_i^*,\tau_i^*))_{i \in N}$ that is consistent with $h$. 
By construction, $\viol_{\I^*}(\bx)-\viol_{\I^*}(\mathbf{z}^*)=D_h(\bx,\mathbf{z}^*)$.
Since $\OPT(\I^*) \le \viol_{\I^*}(\mathbf{z}^*)$, we get $\viol_{\I^*}(\bx) - \OPT(\I^*) \ge D_h(\bx,\mathbf{z}^*) = \max_{\mathbf{z} \in \Delta(S)} D_h(\bx,\mathbf{z})$.
Combining with the upper bound proves the claim.
\end{proof}

This suggests a query rule: choose the next query so as to minimize the largest value that $\min_{\bx}\max_{\mathbf{z}} D_h(\bx,\mathbf{z})$ can take after the answer. Such a rule targets the \textsc{Min-Viol} objective directly, rather than learning each agent's halfspace separately.

\section{Details of the Synthetic Experiments}\label{app:experiments}

The experiments measure how far the worst-case guarantees are from the behavior of the algorithms on random instances.

\textbf{Populations.}
For each of two settings, we generated $20$ independent populations with $n=500$ agents and $m=3$ alternatives. In each population, the planted lottery $\bx$ was drawn uniformly from $\Delta(S)$. Each agent's utilities were $3$ independent uniform values on $[0,1]$, rescaled so that the smallest is $0$ and the largest is $1$; this gives the normalized pair of Section~\ref{sec:prelims}, in which margins are measured. For $425$ agents ($85\%$), $\tau_i$ is $U_i(\bx)$ minus a margin drawn uniformly from $[0.05,0.25]$ in the separated setting and from $[0,0.05)$ in the near-threshold setting, so these agents accept $\bx$. For the other $75$ agents, $\tau_i$ is $U_i(\bx)$ plus an amount drawn uniformly from $[0.05,0.25]$, so they reject $\bx$. An agent whose threshold fell outside $(0,1)$ was drawn again. Hence $B_{0.05}(\bx)=\varnothing$ in the separated setting, whereas in the near-threshold setting $B_{0.05}(\bx)$ consists of all $425$ accepting agents, and $B_{0.02}(\bx)$ of about $169$ of them on average. The planted lottery was optimal in every near-threshold population ($\OPT=75$); in the separated setting, $\OPT$ was between $71$ and $75$.

\textbf{Computing the optimum.}
Since $m=3$, we computed $\OPT$ exactly by checking all candidate lotteries formed by the edges of the simplex and the agents' acceptance boundaries, i.e., the vertex enumeration of Theorem~\ref{thm:minviol_XP}. We compared this exact solver with a mixed-integer program that has one binary variable per agent, indicating whether the agent's constraint may be violated, on all $40$ populations and on all $240$ samples of size $k=120$ described below; the optimal values agreed in every case.

\textbf{Sampling and exact minimization.}
For each population and each sample size $k\in\{30,120,480\}$, we drew $6$ independent samples of $k$ agents uniformly with replacement, minimized the number of rejecting sampled agents exactly, and measured the extra rejection rate $(\viol(\widehat\bx)-\OPT)/n$ on all $500$ agents. As in Section~\ref{sec:statistical}, the normalized pairs of the sampled agents are taken as recovered; we did not simulate the membership queries. Among the minimizers on the sample, vertex enumeration returns the first minimizing vertex it finds; Theorem~\ref{thm:budget-additive} applies to every minimizer. Table~\ref{tab:experiments} reports the mean over all populations and samples and the $95$th percentile. For comparison, with $m=3$ and $k\le 480$, the explicit bound of Theorem~\ref{thm:budget-additive} exceeds $n$ for every $\delta\in(0,1)$, whereas the observed mean extra rejection rate is below $2\%$ from $k=120$ on.

\textbf{The LP.}
Theorem~\ref{thm:shortfall-main} contains an unspecified absolute constant, so we did not evaluate its bound numerically. Instead, we separated the error caused by the LP from the error caused by sampling: we ran the LP with $\eta\in\{0.02,0.05\}$ on all $500$ agents and on the same samples of $k=120$ agents used for exact minimization (Table~\ref{tab:experiments}). In the bound of Theorem~\ref{thm:shortfall-main}, a larger $\eta$ reduces the factor $1+1/\eta$ and the sampling term but enlarges $B_\eta(\bx)$, and the results follow this trade-off. In the separated setting, where $B_{0.05}(\bx)=\varnothing$ for the planted lottery $\bx$, increasing $\eta$ from $0.02$ to $0.05$ lowered the mean extra rejection rate from $0.51\%$ to $0.28\%$ on all agents and from $1.96\%$ to $0.53\%$ on the samples, below the $1.67\%$ of exact minimization on the same samples. This is possible because exact minimization returns a vertex, which typically lies on the boundary of some sampled agents' acceptable sets, whereas the LP favors lotteries that the sampled agents accept with margin. In the near-threshold setting, where $B_{0.05}(\bx)$ contains every accepting agent, increasing $\eta$ raised the rate from $1.86\%$ to $5.22\%$ on all agents and from $3.60\%$ to $9.20\%$ on the samples. Thus the LP can be more accurate than exact minimization when the agents who accept a good lottery do so with a clear margin, but it loses accuracy when many of them are close to their thresholds and $\eta$ is too large. The guarantee of exact minimization does not depend on margins, which is why we prefer it when it is affordable (Section~\ref{subsec:shortfall-lp}).

\textbf{Querying $K$ agents.}
For each population and each $K\in\{10,25,50,100,200,400\}$, we chose $K$ distinct agents uniformly at random $6$ times, minimized the number of violations among them exactly, and measured the additive gap among all $500$ agents. For $K=100$, the mean gap was $8.46$ agents in the separated setting and $7.99$ agents in the near-threshold setting, that is, $2.11\%$ and $2.00\%$ of the worst-case value $n-K=400$. Table~\ref{tab:K-agents} gives all values: the mean gap was between $0.57\%$ and $19.17\%$ of $n-K$, and the largest gap in any single run was $202$ agents, or $41\%$ of $n-K$ (with $K=10$). We also implemented the profile constructed in the proof of Theorem~\ref{thm:coverage} for the algorithm that recovers $K$ fixed agents and minimizes the number of violations among them; on it, the gap was exactly $n-K$ for every tested $K$.

\begin{table}[h]
\caption{Mean additive gap $\viol(\widehat\bx)-\OPT$ of exact minimization over $K$ distinct agents chosen uniformly at random, as a percentage of the worst-case value $n-K$, with $n=500$ and $m=3$.}
\label{tab:K-agents}
\centering
\begin{tabular}{lcccccc}
\toprule
$K$ & $10$ & $25$ & $50$ & $100$ & $200$ & $400$ \\
\midrule
Separated & $13.01$ & $8.62$ & $3.77$ & $2.11$ & $1.23$ & $0.58$ \\
Near-threshold & $19.17$ & $7.79$ & $3.66$ & $2.00$ & $0.88$ & $0.57$ \\
\bottomrule
\end{tabular}
\end{table}

\textbf{Few alternatives.}
With $n=1000$, $m=500$, and $\eta=0.15$, Theorem~\ref{thm:few-alternatives} gives $s_\eta=\lceil\log(2000)/0.045\rceil=169$. In each of $480$ trials, we drew a new population: each agent's utilities were $500$ independent uniform values on $[0,1]$, rescaled as above, and the starting lottery $\bx$ was drawn uniformly from $\Delta(S)$, so that it uses every alternative. Every agent accepts $\bx$: $\tau_i$ is $U_i(\bx)$ minus a margin drawn uniformly from $[\eta,\eta+0.05]$, so that $B_\eta(\bx)=\varnothing$. In a harder setting, the margins of $20\%$ of the agents were instead drawn uniformly from $[0,\eta)$, so that $|B_\eta(\bx)|=200$. For $s\in\{50,169\}$, we drew $s$ alternatives independently from $\bx$, took their empirical distribution as in the proof of Theorem~\ref{thm:few-alternatives}, and counted the agents who reject it. When $B_\eta(\bx)=\varnothing$, sampling $169$ alternatives created no new rejection in any trial; in fact, no agent's expected utility dropped by more than $\eta$ in any trial, whereas the proof guarantees this only with probability at least $1/2$. Sampling only $50$ alternatives created no new rejection in $471$ of the $480$ trials and one new rejection in each of the others. In the harder setting, the mean new rejection rate was $1.19\%$ with $169$ alternatives ($95$th percentile $1.70\%$) and $2.19\%$ with $50$ alternatives ($95$th percentile $3.00\%$), whereas Theorem~\ref{thm:few-alternatives} allows up to $|B_\eta(\bx)|/n=20\%$.

\textbf{Implementation and compute.}
We used Python~3.13 with \href{https://numpy.org/}{NumPy} and \href{https://scipy.org/}{SciPy}, both distributed under the BSD 3-Clause license, and solved the LP and the mixed-integer program with the \href{https://highs.dev/}{HiGHS} solvers (MIT license) in SciPy. In floating point, agent $i$ counts as accepting $\bx$ if $U_i(\bx)\ge\tau_i-10^{-9}$. All random draws use fixed seeds. All instances are small ($n\le 1000$ and $m\le 500$): on a cloud virtual machine with $2$ Intel Xeon vCPUs at $2.1$\,GHz and $8$\,GB of memory, and no GPU, all experiments together take about $2$ minutes, of which about $80$ seconds are spent on the mixed-integer check.
Source files are available at \url{https://github.com/nicteh/Learning-Lotteries-Minimal-Stakeholder-Violations}.

\end{document}